\documentclass[11pt]{article}

\usepackage[utf8]{inputenc}
\usepackage[T1]{fontenc}
\usepackage{amsmath,amssymb,amsthm,amsfonts}

\usepackage{graphicx}

\let\origincludegraphics\includegraphics
\renewcommand{\includegraphics}[2][]{%
  \IfFileExists{#2}{\origincludegraphics[#1]{#2}}{%
    \begingroup
      \setlength{\fboxsep}{2pt}%
      \fbox{\footnotesize\texttt{Missing figure: #2}}%
    \endgroup
  }%
}

\graphicspath{{figures/}{../figures/}}
\newcommand{\safeincludegraphics}[2][]{%
  \IfFileExists{#2}{%
    \includegraphics[#1]{#2}%
  }{%
    \begingroup
      \setlength{\fboxsep}{6pt}%
      \fbox{\parbox{0.9\linewidth}{\centering\textbf{Missing figure}\\\ttfamily\detokenize{#2}}}%
    \endgroup
  }%
}

\usepackage{grffile}
\usepackage[square,numbers,sort&compress]{natbib}
\usepackage{float}
\usepackage{booktabs}
\usepackage{array}
\usepackage{xcolor}
\usepackage[margin=1in]{geometry}

\usepackage{caption}
\usepackage{subcaption}

\usepackage{algorithm}
\usepackage{algpseudocode}
\usepackage{etoolbox}
\usepackage{bbm}

\usepackage{hyperref}

\let\AlgoState\State
\AtBeginEnvironment{algorithmic}{%
  \let\State\AlgoState
}
\algrenewcommand\algorithmicrequire{\textbf{Input:}}
\algrenewcommand\algorithmicensure{\textbf{Output:}}

\theoremstyle{plain}
\newtheorem{theorem}{Theorem}[section]
\newtheorem{proposition}[theorem]{Proposition}

\newtheorem{corollary}[theorem]{Corollary}

\theoremstyle{definition}
\newtheorem{definition}[theorem]{Definition}
\newtheorem{axiom}[theorem]{Axiom}

\theoremstyle{remark}

\renewcommand{\State}{\mathcal{S}}
\newcommand{\InputSpace}{\mathcal{X}}
\newcommand{\LabelSpace}{\mathcal{Y}}
\newcommand{\FormSpace}{\mathcal{F}}
\newcommand{\HypSpace}{\mathcal{H}}
\newcommand{\ActionSpace}{\mathcal{A}}
\newcommand{\ObsSpace}{\mathcal{O}}

\newcommand{\Void}{\varnothing}
\newcommand{\Mark}{\mathord{\text{\textcircled{\tiny M}}}}
\newcommand{\Cross}[1]{\overline{#1}}
\newcommand{\BigJoin}[1]{\bigsqcup #1}
\newcommand{\Atom}[1]{a_{#1}}
\newcommand{\ReEntry}[1]{@#1}

\newcommand{\Energy}{E}
\newcommand{\Complexity}{C}
\newcommand{\Loss}{\mathcal{L}}
\newcommand{\Jobj}{J}
\newcommand{\Fisher}{\mathbf{F}}
\newcommand{\Manifold}{\mathcal{M}}

\newcommand{\lamc}{\lambda_c}
\newcommand{\lame}{\lambda_e}

\newcommand{\eval}{\text{eval}}

\newcommand{\E}{\mathbb{E}}
\newcommand{\Prob}{\mathbb{P}}
\newcommand{\ind}{\mathbf{1}}

\definecolor{teleoblue}{RGB}{0,114,178}
\definecolor{teleoorange}{RGB}{230,159,0}
\definecolor{teleogreen}{RGB}{0,158,115}
\definecolor{teleored}{RGB}{213,94,0}

\title{%
\textbf{Teleodynamic Learning: A New Paradigm For Interpretable AI}\\[0.5em]
\large Learning as Navigation in Coupled Self-Organizing Systems
}

\author{
\textit{Enrique ter Horst, Juan Zambrano}\\
}

\date{}

\begin{document}

\maketitle

\begin{abstract}
We introduce \emph{Teleodynamic Learning}, a paradigm shift in machine learning that treats learning not as the minimization of a static objective, but as \emph{the emergence and stabilization of functional organization under constraint}. Just as many of AI’s decisive ideas were imported from living systems---neuronal computation, hierarchical sensory processing, energy-based memory---teleodynamic learning formalizes what biology makes unavoidable: adaptive intelligence co-evolves \emph{what it can represent}, \emph{how it fits parameters}, and \emph{which changes are affordable}. Concretely, we model learning as navigation in a constrained dynamical system with two coupled timescales: \emph{inner dynamics} (continuous parametric adaptation) and \emph{outer dynamics} (discrete structural modification), linked by an \emph{endogenous resource variable} that both emerges from and regulates the trajectory. This coupling yields three central phenomena absent from standard optimization accounts: (i) \emph{emergent stabilization} without externally imposed stopping criteria, (ii) \emph{phase-structured behavior} (under-structuring $\rightarrow$ teleodynamic growth $\rightarrow$ over-structuring) diagnosable through dynamical signatures rather than heuristic tuning, and (iii) convergence guarantees grounded in the geometry of the parameter manifold (via natural-gradient structure) rather than convexity.

We instantiate the paradigm in the \emph{Distinction Engine} (DE11), a teleodynamic learner grounded in Spencer-Brown’s \emph{Laws of Form} \citep{spencer-brown1969}, information geometry, and tropical optimization. On standard benchmarks, DE11 achieves 93.3\% test accuracy on IRIS (vs.\ 91.1\% logistic regression), 92.6\% on WINE, and 94.7\% on Breast Cancer, while producing interpretable logical rules \citep{rudin2019stop} that \emph{arise endogenously} from the coupled dynamics rather than being imposed by design. Overall, teleodynamic learning unifies regularization, architecture search, and resource-bounded inference as manifestations of a single principle: learning as the co-evolution of structure, parameters, and resources under constraint, opening a thermodynamically grounded route to adaptive and interpretable AI.
\end{abstract}

\section{Introduction: Beyond Optimization}
\label{sec:introduction}


The history of machine learning is, in a quiet but decisive sense, a history of \emph{biological borrowing}. Some of the field’s most consequential abstractions were not invented from first principles of optimization; they were extracted from how living systems compute under constraint. The formal neuron begins as a stylized account of nervous activity \cite{mcculloch1943logical}, synaptic plasticity is canonized as a learning postulate \cite{hebb1949organization}, and early supervised learning is framed as a brain-inspired statistical device \cite{rosenblatt1958perceptron}. Modern representation learning continues this lineage: deep networks are trained efficiently by error backpropagation \cite{rumelhart1986backprop}, but the architectural priors that make them work at scale often echo biological organization.

Convolutional neural networks provide a particularly crisp example. Their core inductive biases—local receptive fields, hierarchical feature composition, and (approximate) invariances—track long-standing findings in sensory neurophysiology \cite{hubel1962receptive}, were operationalized in early hierarchical vision models \cite{fukushima1980neocognitron}, and became a practical learning system through gradient-based training \cite{lecun1998document}. A parallel thread links learning to physics and physiology: associative memory and energy landscapes \cite{hopfield1982neural}, stochastic neural computation via statistical mechanics \cite{ackley1985boltzmann}, and modern training heuristics for energy-based and latent-variable models \cite{hinton2002contrastive}. The deep learning resurgence further consolidated these ideas through principled initialization and layerwise representation learning \cite{hinton2006dbn}, dimensionality-reducing autoencoders \cite{hinton2006autoencoders}, and a unifying synthesis of the modern “deep” paradigm \cite{lecun2015deep}. Even reinforcement learning’s canonical signal—the prediction error—has a direct neurobiological analogue in reward circuitry \cite{schultz1997reward}, while gated memory mechanisms explicitly target the computational role of persistent internal state \cite{hochreiter1997lstm}. At the limit of formalization, universal agent models aim to express intelligence as sequential decision-making under uncertainty and description length \cite{hutter2005uai}.

These connections are not mere origin stories. They point to a persistent mismatch between what living systems demonstrably do and what standard learning theory typically assumes. Biological learning is not simply parameter fitting inside a fixed hypothesis class; it is a coupled process in which \emph{structure}, \emph{parameters}, and \emph{resources} co-determine one another over time. In theoretical biology, this end-directedness is not treated as mysticism but as an emergent consequence of constrained dynamics: teleological language is rehabilitated as a disciplined account of functional organization \cite{mayr1961cause}, organizational closure \cite{mossio2010closure,montevil2015closure}, and the conditions under which systems become legitimately describable as “for” something \cite{mossio2017teleology}. Relatedly, the free-energy perspective frames adaptive behavior as maintaining viable organization under uncertainty \cite{friston2010free,friston2013life}, while nonequilibrium statistical physics clarifies how driven systems can self-organize through dissipation \cite{england2013selfreplication,england2015dissipative}. Deacon’s teleodynamics makes the key conceptual step explicit: end-directed dynamics arise when constraints and processes mutually sustain one another, yielding a system whose “aims” are realized as dynamical necessities rather than externally imposed goals \cite{deacon2011incomplete}.

Yet the dominant conceptualization of machine learning still rests on an implicit axiom: \emph{learning is optimization}. A model is defined, a loss function is specified, and the learning algorithm seeks parameters that minimize this function. This framing, inherited from statistical estimation theory, has proven remarkably productive. Yet it carries an assumption so pervasive as to be invisible: that the structure of the hypothesis class is fixed \emph{a priori}, and that the only quantity subject to adaptation is the parameter vector within that fixed structure.

This paper argues that this conceptualization is not merely incomplete but categorically insufficient for understanding adaptive systems. We develop an alternative paradigm---\emph{teleodynamic learning}---in which structure, parameters, and computational resources co-evolve under mutual constraint. The resulting system does not optimize an objective; it navigates a dynamical landscape whose topology emerges from the coupling itself. In this sense, our proposal is an “energy method” in the biological meaning of the phrase: not a hand-tuned regularizer appended to a loss, but an endogenous accounting of viability, expenditure, and constraint maintenance that determines which structural and parametric moves are even available.

\subsection{The Structural Problem}

Consider a learner that must simultaneously determine \emph{what} to represent and \emph{how} to parameterize that representation. In neural architecture search, one optimizes over architectures and then over parameters within each architecture. In program synthesis, one searches over program structures and then grounds those structures in data. In rule learning, one grows a hypothesis space and then assigns confidences to each hypothesis. In each case, two qualitatively different processes interleave: discrete structural changes and continuous parametric adaptation.

Standard optimization theory handles these processes poorly. The composition of discrete and continuous search spaces lacks the regularity properties (convexity, smoothness, connectedness) on which convergence guarantees depend. The typical response is to separate the two: perform structural search first (by heuristics, evolutionary methods, or reinforcement learning), then optimize parameters. This separation is mathematically convenient but biologically and computationally implausible. Living systems do not freeze their structure before tuning their parameters.

The teleodynamic paradigm rejects this separation. Structure and parameters evolve together, but on different timescales and through different mechanisms. Crucially, these two processes are \emph{coupled} through a third quantity---an endogenous resource---that emerges from the dynamics and constrains further change.

\subsection{The Resource Problem}

Every learning system operates under resource constraints. These constraints are typically imposed externally: a training budget, a parameter count limit, a wallclock deadline. The learning algorithm is unaware of these constraints and optimizes as if resources were infinite; the external environment truncates the process.

This arrangement is theoretically incoherent. If resources genuinely constrain the solution, they should enter the objective. But how? Adding a regularization term converts a hard constraint into a soft preference. Lagrangian methods require knowing the constraint value in advance. And neither approach makes the resource \emph{endogenous}---a quantity that the system itself generates and depletes.

In teleodynamic learning, the resource is not a fixed budget but a dynamical variable. Actions consume or replenish resources. Resource pressure gates which actions are viable. The system cannot simply ``spend'' its way to higher accuracy; it must navigate a trajectory where performance, complexity, and resources remain in dynamic tension.

\subsection{The Emergence Problem}

The most striking feature of biological learning systems is their capacity for self-organization: the emergence of stable, functional structure without explicit architectural specification. This emergence is not mysterious; it arises from the coupling of multiple dynamical processes under constraint. But it is precisely this coupling that standard optimization lacks.

A regularized objective like $\Loss + \lambda \cdot \text{Complexity}$ does not yield emergence. The complexity term provides a gradient that uniformly opposes growth; the system finds the fixed point where this gradient balances the loss gradient. There is no dynamics, no trajectory, no history dependence. The solution is fully determined by the objective and the initial condition.

Teleodynamic systems exhibit genuine emergence because their dynamics are \emph{irreducibly temporal}. The system's current state depends on its history---which actions were taken, which resources were spent, which structures crystallized. The final structure is not the minimizer of any static function; it is the attractor of a dynamical process.

\subsection{Core Contributions}

This paper makes the following contributions:

\begin{enumerate}
    \item \textbf{Paradigm.} We formalize teleodynamic learning as a paradigm distinct from static optimization, characterized by two-timescale dynamics, endogenous resource coupling, and emergent stabilization (\S\ref{sec:paradigm}).
    
    \item \textbf{Framework.} We provide a complete mathematical framework instantiating this paradigm, grounded in Spencer-Brown's Laws of Form for logical structure, information geometry for parametric adaptation, and coalgebraic semantics for compositional dynamics (\S\ref{sec:framework}).
    
    \item \textbf{Dynamics.} We prove that the inner (parametric) dynamics converge under natural gradient descent with diagonal Fisher information, independent of whether structure stabilizes (\S\ref{sec:dynamics}).
    
    \item \textbf{Emergence.} We demonstrate that structural dynamics self-terminate without imposed stopping criteria, and characterize the phase structure (under-structuring, teleodynamic growth, over-structuring) through which systems evolve (\S\ref{sec:phases}).
    
    \item \textbf{Instantiation.} We instantiate the paradigm in a concrete system (DE11) that achieves competitive performance on standard benchmarks (93.3\% IRIS, 92.6\% WINE, 94.7\% Breast Cancer) while producing interpretable logical rules (\S\ref{sec:experiments}).
\end{enumerate}

\subsection{What This Paper Is Not}

Teleodynamic learning is not a new algorithm. It is a paradigm---a way of conceptualizing the learning problem that determines what questions are meaningful and what answers are satisfactory. Many existing approaches can be reinterpreted teleodynamically (regularization as resource pressure, architecture search as structural dynamics, active learning as information-seeking behavior). The contribution is not to add a technique to the toolbox but to reconfigure the toolbox itself.

Nor is this paper a claim of state-of-the-art performance. We do not compete on ImageNet or language modeling benchmarks. Our experiments use small, interpretable datasets precisely because they allow us to study the \emph{dynamics} of learning---the trajectory through structure-parameter-resource space---rather than merely the endpoint.

\subsection{Paper Organization}

Section~\ref{sec:paradigm} presents the teleodynamic paradigm in full generality, independent of any particular instantiation. Section~\ref{sec:framework} develops the mathematical framework: Forms for logical structure, Manifolds for parametric geometry, and State coalgebras for compositional dynamics. Section~\ref{sec:dynamics} establishes convergence properties and the separation of timescales. Section~\ref{sec:experiments} presents empirical results. Section~\ref{sec:phases} analyzes the phase structure of teleodynamic trajectories. Section~\ref{sec:discussion} situates the work relative to related approaches and discusses limitations. Section~\ref{sec:conclusion} reflects on the broader implications.

\section{The Teleodynamic Paradigm}
\label{sec:paradigm}


Before presenting the mathematical framework, we articulate the teleodynamic paradigm in its full generality. This section establishes the conceptual foundation; subsequent sections provide the formal apparatus.

\subsection{The Five Commitments}

A system instantiates teleodynamic learning if and only if it satisfies the following five commitments. These are not implementation choices; they are definitional. A system violating any commitment is not teleodynamic, regardless of other similarities.

\begin{definition}[Two-Timescale Dynamics]
\label{def:two-timescale}
The system maintains two qualitatively distinct adaptive processes:
\begin{enumerate}
    \item \textbf{Inner dynamics} (parametric adaptation): Continuous or quasi-continuous updates to parameters within a fixed structure. These run whenever predictive signal exists.
    \item \textbf{Outer dynamics} (structural modification): Discrete interventions that change the hypothesis class itself. These occur sparingly and are explicitly evaluated.
\end{enumerate}
Crucially, inner dynamics \emph{continue} after outer dynamics cease. Structure can freeze while parameters remain plastic.
\end{definition}

\begin{definition}[Endogenous Resource]
\label{def:resource}
The system maintains an internal scalar $\Energy \in \mathbb{R}^+$ (energy, budget, capacity) with the following properties:
\begin{enumerate}
    \item $\Energy$ is \emph{not} a fixed constraint or hyperparameter specified externally.
    \item All learning actions consume or replenish $\Energy$.
    \item The current value of $\Energy$ directly influences which actions are viable.
    \item $\Energy$ couples the inner and outer dynamics: structural actions are costly in energy.
\end{enumerate}
\end{definition}

\begin{definition}[Local Teleodynamic Objective]
\label{def:teleo-obj}
Structural actions are selected by minimizing a local objective:
\begin{equation}
    \Jobj_a = \Loss(s', x, y) + \lamc \cdot \Delta\Complexity + \lame \cdot \Delta\Energy
    \label{eq:local-J}
\end{equation}
where $\Loss$ is the predictive loss on the current sample, $\Delta\Complexity = \Complexity(s') - \Complexity(s)$ is the change in structural complexity, and $\Delta\Energy = \Energy(s) - \Energy(s')$ is the energy consumed. This objective is evaluated \emph{per action}, not globally optimized. No assumption of convexity, stationarity, or global convergence is made.
\end{definition}

\begin{definition}[Emergent Structural Halt]
\label{def:halt}
The system does not impose early stopping, fixed budgets, or hard architectural limits from outside. Structural exploration must \emph{self-terminate} when further change is no longer justified by the objective. Specifically:
\begin{enumerate}
    \item The system possesses a ``noop'' action that maintains current structure.
    \item Noop is selected when its $\Jobj$ score exceeds all structural alternatives.
    \item After structural halt, parametric adaptation continues via inner dynamics.
\end{enumerate}
\end{definition}

\begin{definition}[Phase Structure]
\label{def:phases}
The system's learning trajectory exhibits distinct dynamical regimes. At minimum, the system must be capable of manifesting:
\begin{enumerate}
    \item \textbf{Under-structuring:} Insufficient hypotheses to achieve low loss.
    \item \textbf{Teleodynamic growth:} Structural expansion balanced by complexity and energy costs.
    \item \textbf{Over-structuring:} Excessive structure with diminishing returns.
\end{enumerate}
A valid teleodynamic system must be diagnosable via phase diagrams or equivalent dynamical analysis.
\end{definition}

\subsection{What Teleodynamic Learning Is Not}

The paradigm is easily confused with related but distinct approaches. We clarify the distinctions.

\paragraph{Not Static Optimization.}
A regularized objective $\min_\theta \Loss(\theta) + \lambda R(\theta)$ is not teleodynamic. The regularization term provides a static trade-off encoded in the hyperparameter $\lambda$. There is no dynamics: the solution is the fixed point of gradient descent, independent of trajectory. In contrast, teleodynamic systems are path-dependent; the same initial condition and data can yield different final structures depending on the sequence of actions.

\paragraph{Not Neural Architecture Search.}
NAS methods search over discrete architecture choices using reinforcement learning, evolutionary algorithms, or gradient-based relaxations. These are outer-loop procedures that wrap an inner optimization. Teleodynamic learning does not separate inner and outer loops in this way; both operate simultaneously, coupled through the resource. Moreover, NAS typically lacks an endogenous resource---the search budget is specified externally.

\paragraph{Not Regularization in Disguise.}
The coefficients $\lamc$ and $\lame$ in Equation~\eqref{eq:local-J} might suggest regularization. But the objective is evaluated \emph{locally} (per action, per sample), not globally (over the entire dataset and all possible parameters). A regularization term pulls toward a fixed point; the teleodynamic objective gates discrete transitions. These are categorically different operations.

\paragraph{Not Minimum Description Length.}\label{par:not-mdl}\citep{rissanen1989stochastic,grunwald2007minimum}
MDL~\citep{rissanen1989stochastic,grunwald2007minimum} provides a principled trade-off between model complexity and data fit. Teleodynamic learning can be viewed as implementing an MDL-like principle, but with crucial differences: (1) the complexity measure is emergent from the dynamics, not specified a priori; (2) the optimization is local and greedy, not global; (3) the resource variable has no MDL analog.

\subsection{The Dynamical Perspective}

The teleodynamic paradigm adopts a fundamentally dynamical perspective on learning. Rather than asking ``What is the optimal hypothesis?'' we ask ``What trajectory does the system follow through hypothesis space?'' The endpoint is determined by the dynamics; it is not an optimum of any pre-specified function.

This perspective has precedent. Hopfield networks converge to attractor states that minimize an energy function, but the function is derived from the dynamics, not imposed. Free-energy principles in neuroscience posit that perception and action minimize variational free energy, which emerges from the coupling of internal and external states. Thermodynamic computing relates computation to physical constraints like Landauer's limit. Teleodynamic learning synthesizes these threads into a coherent paradigm for machine learning.

\begin{figure}[t]
    \centering
    \safeincludegraphics[width=\textwidth]{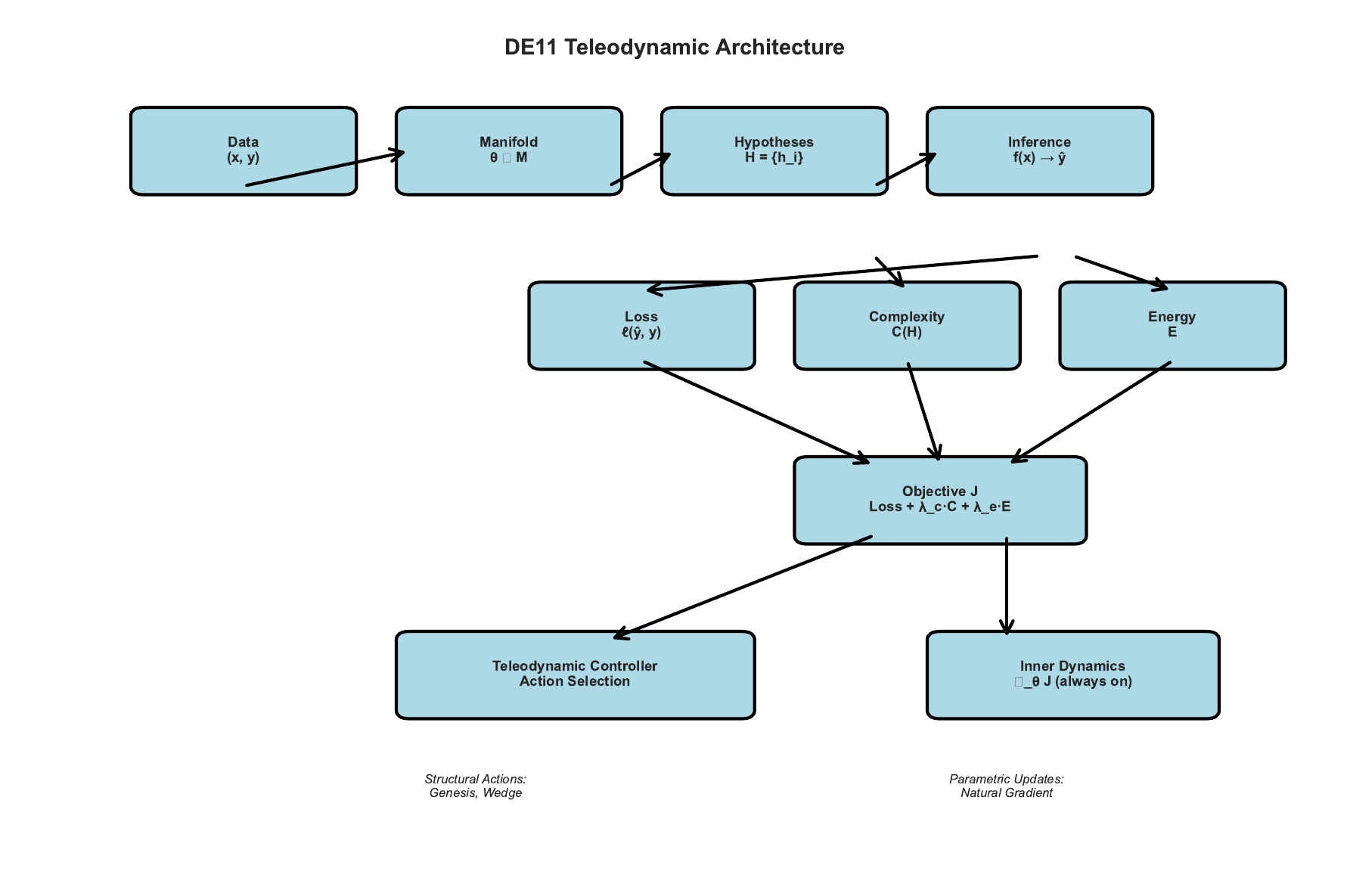}
    \caption{\textbf{Teleodynamic system architecture.} The system state $s \in \State$ comprises structure (hypotheses $\mathcal{H}$), parameters ($\theta \in \Manifold$), energy $\Energy$, and history $\tau$. Observations $(x, y)$ trigger the coalgebraic step function $\gamma$, producing an observation $o \in \ObsSpace$ and a successor state $s'$. Structural actions (genesis, wedge) and parametric updates (natural gradient on the manifold) compete; the local teleodynamic objective $\Jobj$ gates action selection. The resource $\Energy$ couples both dynamics: structural actions cost energy, while predictive success replenishes it.}
    \label{fig:architecture}
\end{figure}

\subsection{The Beneficiary Structure}

The concept of ``teleodynamics~\citep{deacon2011incomplete}'' originates in Terrence Deacon~\citep{deacon2011incomplete}'s \emph{Incomplete Nature}, where it describes systems that exhibit purpose-like behavior without external design. A teleodynamic system is characterized by constraint closure: the system constrains itself in ways that perpetuate its own existence. The key feature is a \emph{beneficiary}---an entity that benefits from the system's behavior and whose persistence explains that behavior.

In teleodynamic learning, the beneficiary is the hypothesis ensemble itself. Hypotheses that perform well receive energy (through correct predictions) and persist; hypotheses that perform poorly drain energy and are pruned. The ensemble acts to maximize its own fitness, not through explicit optimization but through the differential survival of its components. This is not merely an analogy to evolution; it is a precise instantiation of the beneficiary structure in a computational setting.

The three stages of Deaconian teleodynamics map onto our framework:
\begin{enumerate}
    \item \textbf{Thermodynamic:} Raw gradient flow on the loss landscape (inner dynamics without resource coupling).
    \item \textbf{Morphodynamic:} Self-organizing structural change (outer dynamics generating new hypotheses).
    \item \textbf{Teleodynamic:} Constraint closure through resource coupling (the mature system maintaining itself).
\end{enumerate}

A system that has reached teleodynamic maturity exhibits homeostasis: perturbations trigger responses that restore the prior state. Unexpected errors increase structural pressure, potentially creating new hypotheses; prolonged success allows structure to freeze as parameters refine. This is not equilibrium (a static balance) but homeostasis (an actively maintained dynamic balance).

\subsection{Formal Definition}

We now state the teleodynamic learning paradigm as a mathematical object.

\begin{definition}[Teleodynamic Learning System]
\label{def:tls}
A \emph{teleodynamic learning system} is a tuple $(\State, \InputSpace, \LabelSpace, \ActionSpace, \gamma, \Jobj, \Energy_0)$ where:
\begin{itemize}
    \item $\State$ is the state space, with each state $s \in \State$ decomposing as $s = (\mathcal{H}, \theta, \Energy, \tau)$ for hypothesis set $\mathcal{H}$, parameters $\theta$, energy $\Energy$, and history $\tau$.
    \item $\InputSpace$ and $\LabelSpace$ are input and label spaces.
    \item $\ActionSpace$ is a finite set of actions, partitioned into structural actions $\ActionSpace_{\text{struct}}$ and parametric actions $\ActionSpace_{\text{param}}$.
    \item $\gamma: \State \times \InputSpace \times \LabelSpace \to \ObsSpace \times \State$ is the coalgebraic step function.
    \item $\Jobj: \State \times \ActionSpace \times \InputSpace \times \LabelSpace \to \mathbb{R}$ is the local teleodynamic objective.
    \item $\Energy_0 > 0$ is the initial energy.
\end{itemize}
The step function $\gamma$ implements:
\begin{enumerate}
    \item Parametric update (inner dynamics) if predictive signal exists.
    \item Action selection $a^* = \arg\min_{a \in \ActionSpace} \Jobj(s, a, x, y)$ subject to energy constraint.
    \item State transition $s' = \text{apply}(s, a^*)$ with $\Energy' = \Energy - \text{cost}(a^*)$.
\end{enumerate}
\end{definition}

This definition is abstract by design. The framework admits many instantiations: rule systems, neural networks, program synthesizers, or hybrid systems. The paradigmatic commitment is to the structure of coupling, not the substrate of representation.

\section{Mathematical Framework}
\label{sec:framework}


We instantiate the teleodynamic paradigm through a specific mathematical framework based on three pillars: Spencer-Brown's Laws of Form for logical structure, information geometry for parametric adaptation, and coalgebraic semantics for compositional dynamics.

\subsection{Forms: The Calculus of Distinctions}

The hypothesis language is derived from Spencer-Brown's \emph{Laws of Form} (1969), a calculus for distinctions that predates and generalizes Boolean algebra. We extend the classical calculus with probabilistic semantics and gradients.

\begin{definition}[Form Algebra]
\label{def:forms}
The set of \emph{forms} $\FormSpace$ is the smallest set closed under:
\begin{align}
    f \in \FormSpace ::=\ & \Void & \text{(void: the unmarked state)} \\
    |\ & \Mark & \text{(mark: the marked state)} \\
    |\ & \Atom{i} & \text{(atom: primitive distinction indexed by } i \text{)} \\
    |\ & \Cross{f} & \text{(cross: complement/negation)} \\
    |\ & \BigJoin{S} & \text{(call: disjunction over set } S \subseteq \FormSpace \text{)}
\end{align}
\end{definition}

The classical Laws of Form specify two axioms that govern this algebra:

\begin{axiom}[Condensation]
$\Cross{\Cross{f}} = f$ for all $f \in \FormSpace$.
\end{axiom}

\begin{axiom}[Cancellation]
$\BigJoin{\{f, \Cross{f}\}} = \Mark$ for all $f \in \FormSpace$.
\end{axiom}

These axioms are the computational primitives of classical logic: double negation elimination and the law of excluded middle. Our extension lies in the probabilistic semantics.

\begin{definition}[Soft Evaluation]
\label{def:eval-soft}
Let $\Manifold$ be a parameter manifold and $x \in \InputSpace$ an input. The \emph{soft evaluation} $\eval: \FormSpace \times \InputSpace \times \Manifold \to [0, 1]$ is defined recursively:
\begin{align}
    \eval(\Void, x, \theta) &= 0 \\
    \eval(\Mark, x, \theta) &= 1 \\
    \eval(\Atom{i}, x, \theta) &= \sigma(\theta_i^\top x + b_i) = \frac{1}{1 + e^{-(\theta_i^\top x + b_i)}} \label{eq:atom-sigmoid} \\
    \eval(\Cross{f}, x, \theta) &= 1 - \eval(f, x, \theta) \\
    \eval(\BigJoin{S}, x, \theta) &= 1 - \prod_{f \in S} (1 - \eval(f, x, \theta)) \label{eq:noisy-or}
\end{align}
\end{definition}

The semantics of atoms (Equation~\ref{eq:atom-sigmoid}) interpret each atom as a halfspace in input space, with soft boundaries given by the sigmoid. The semantics of Call (Equation~\ref{eq:noisy-or}) implements noisy-OR: the probability that at least one disjunct holds, assuming independence. This generalizes classical OR to the probabilistic domain while preserving the Laws of Form axioms.

\begin{proposition}[Axiom Preservation]
Under soft evaluation, the Laws of Form axioms hold in the limit of sharp (non-probabilistic) semantics:
\begin{enumerate}
    \item $\lim_{\beta \to \infty} \eval(\Cross{\Cross{f}}, x, \beta\theta) = \eval(f, x, \beta\theta)$
    \item $\lim_{\beta \to \infty} \eval(\BigJoin{\{f, \Cross{f}\}}, x, \beta\theta) = 1$
\end{enumerate}
where $\beta\theta$ scales all parameters, sharpening the sigmoids.
\end{proposition}

\begin{definition}[Complexity]
\label{def:complexity}
The \emph{complexity} $\Complexity: \FormSpace \to \mathbb{R}^+$ is defined recursively:
\begin{align}
    \Complexity(\Void) = \Complexity(\Mark) &= 0 \\
    \Complexity(\Atom{i}) &= 1 \\
    \Complexity(\Cross{f}) &= 1 + \Complexity(f) \\
    \Complexity(\BigJoin{S}) &= \sum_{f \in S} \Complexity(f)
\end{align}
\end{definition}

This complexity measure counts the total number of primitive distinctions (atoms and crosses) in a form. It is additive over disjunctions and accumulates through negations, reflecting the intuition that more complex logical formulas require more distinctions to specify.

\begin{definition}[Gradient Backpropagation]
\label{def:grad-soft}
The gradient $\nabla_\theta \eval(f, x, \theta)$ propagates through the form structure:
\begin{align}
    \nabla_\theta \eval(\Void, x, \theta) &= 0 \\
    \nabla_\theta \eval(\Mark, x, \theta) &= 0 \\
    \nabla_\theta \eval(\Atom{i}, x, \theta) &= p_i(1 - p_i) \cdot [x; 1] \quad \text{where } p_i = \eval(\Atom{i}, x, \theta) \\
    \nabla_\theta \eval(\Cross{f}, x, \theta) &= -\nabla_\theta \eval(f, x, \theta) \\
    \nabla_\theta \eval(\BigJoin{S}, x, \theta) &= \sum_{f \in S} \frac{1 - \eval(\BigJoin{S}, x, \theta)}{1 - \eval(f, x, \theta)} \cdot \nabla_\theta \eval(f, x, \theta)
\end{align}
\end{definition}

The gradient through the noisy-OR (Call) implements the chain rule: each child's gradient is weighted by the marginal contribution of that child to the overall probability.

\subsection{Manifold: Information Geometry}

Parameters live on a statistical manifold equipped with the Fisher information metric. This geometry is not merely decorative; it determines the natural gradient, which provides convergence guarantees independent of parameterization.

\begin{definition}[Parameter Manifold]
\label{def:manifold}
The parameter manifold $\Manifold$ is the space of atom parameters $\theta = \{(\theta_i, b_i)\}_{i=1}^N$ where $\theta_i \in \mathbb{R}^d$ and $b_i \in \mathbb{R}$. We write $\theta_i^{\text{ext}} = [\theta_i; b_i] \in \mathbb{R}^{d+1}$ for the extended parameter vector including the bias.
\end{definition}

\begin{definition}[Fisher Information Matrix]
\label{def:fisher}
For a probabilistic model $p(y | x, \theta)$, the Fisher information matrix is:
\begin{equation}
    \Fisher(\theta) = \E_{p(x)p(y|x,\theta)}\left[ \nabla_\theta \log p(y|x,\theta) \nabla_\theta \log p(y|x,\theta)^\top \right]
    \label{eq:fisher-def}
\end{equation}
\end{definition}

The Fisher matrix encodes how sensitive the predictive distribution is to parameter changes. Parameters that have large effect on predictions have large Fisher information; parameters that are redundant or saturated have small Fisher information. The natural gradient normalizes updates by this sensitivity:

\begin{definition}[Natural Gradient]
\label{def:nat-grad}
The \emph{natural gradient} of a loss $\Loss(\theta)$ is:
\begin{equation}
    \tilde{\nabla}_\theta \Loss = \Fisher(\theta)^{-1} \nabla_\theta \Loss
    \label{eq:nat-grad}
\end{equation}
The natural gradient update is $\theta_{t+1} = \theta_t - \eta \tilde{\nabla}_\theta \Loss$.
\end{definition}

For multiclass classification with $K$ classes, the model produces logits $u_k(x, \theta) = \sum_{h: \text{outcome}(h)=k} \alpha_h \cdot \eval(f_h, x, \theta)$ and probabilities via softmax:
\begin{equation}
    p(y = k | x, \theta) = \frac{\exp(u_k(x, \theta))}{\sum_{j=1}^K \exp(u_j(x, \theta))}
    \label{eq:softmax}
\end{equation}

The cross-entropy~\citep{amari2016information} loss is $\Loss(\theta; x, y) = -\log p(y | x, \theta)$.

\begin{proposition}[Diagonal Fisher Approximation]
\label{prop:diag-fisher}
Under the diagonal approximation $\Fisher(\theta) \approx \text{diag}(\Fisher_{ii}(\theta))$, the empirical Fisher for the softmax cross-entropy loss can be estimated online:
\begin{equation}
    \hat{\Fisher}_{ii}^{(t)} = \beta \hat{\Fisher}_{ii}^{(t-1)} + (1 - \beta) \left( \frac{\partial \Loss}{\partial \theta_i} \right)^2
    \label{eq:fisher-ema}
\end{equation}
where $\beta \in (0, 1)$ is a decay factor (typically 0.9--0.99).
\end{proposition}

This is the standard empirical Fisher used in AdaGrad, RMSProp, and Adam, reinterpreted through the lens of information geometry. The DE11 implementation uses this approximation with $\beta = 0.9$.

\subsection{State: Coalgebraic Semantics}

The system state and its evolution are formalized coalgebraically. A coalgebra captures the observable behavior of a system without committing to its internal representation.

\begin{definition}[State Space]
\label{def:state}
The state space is:
\begin{equation}
    \State = \HypSpace^* \times \Manifold \times \mathbb{R}^+ \times (\InputSpace \times \LabelSpace)^* \times \text{Registry}
\end{equation}
A state $s = (\mathcal{H}, \theta, \Energy, \tau, R)$ comprises:
\begin{itemize}
    \item $\mathcal{H} = (h_1, \ldots, h_n)$: a tuple of hypotheses, each $h_i = (f_i, y_i, r_i, m_i, \alpha_i)$ with form $f_i$, outcome $y_i$, reliability $r_i$, memory $m_i$, and weight $\alpha_i$.
    \item $\theta \in \Manifold$: the parameter manifold.
    \item $\Energy \in \mathbb{R}^+$: the endogenous resource.
    \item $\tau \in (\InputSpace \times \LabelSpace)^*$: the experience history.
    \item $R: \text{Id} \to \FormSpace$: the registry for compressed subforms.
\end{itemize}
\end{definition}

\begin{definition}[Hypothesis]
\label{def:hypothesis}
A hypothesis $h = (f, y, r, m, \alpha)$ encapsulates:
\begin{itemize}
    \item $f \in \FormSpace$: the logical form (region definition).
    \item $y \in \LabelSpace$: the predicted outcome when $f$ fires.
    \item $r \in [0, 1]$: the reliability (exponential moving average of correctness).
    \item $m = (m^+, m^-)$: indices into history of positive and negative examples.
    \item $\alpha \in \mathbb{R}^+$: the output weight for logit contribution.
\end{itemize}
The \emph{tropical cost} of a hypothesis is $\text{cost}(h) = \Complexity(f) + 5(1 - r)$, balancing complexity against reliability.
\end{definition}

\begin{definition}[Observation Space]
\label{def:obs}
An observation $o \in \ObsSpace$ is:
\begin{equation}
    o = (\hat{y}, c, w, \delta\Energy, a)
\end{equation}
comprising the prediction $\hat{y}$, confidence $c \in [0,1]$, winner hypothesis $w$, energy change $\delta\Energy$, and action taken $a$.
\end{definition}

\begin{definition}[Coalgebraic Step Function]
\label{def:step}
The step function $\gamma: \State \times \InputSpace \times \LabelSpace \to \ObsSpace \times \State$ is:
\begin{equation}
    \gamma(s, x, y) = (o, s')
\end{equation}
where the transition is computed as:
\begin{enumerate}
    \item \textbf{Inference:} Compute softmax probabilities over classes via Equation~\ref{eq:softmax}; select prediction $\hat{y} = \arg\max_k p(y=k|x,\theta)$.
    \item \textbf{Tropical selection:} Among hypotheses with $\eval(f_h, x, \theta) > 0.5$, select winner $w = \arg\min_h \text{cost}(h)$.
    \item \textbf{Reward:} Compute energy change $\delta\Energy = r_{\text{correct}} \cdot \ind[\hat{y} = y] + r_{\text{wrong}} \cdot \ind[\hat{y} \neq y]$.
    \item \textbf{History update:} Append $(x, y)$ to $\tau$.
    \item \textbf{Parametric update (inner dynamics):} Apply natural gradient to $\theta$ using gradients from $\Loss(\theta; x, y)$.
    \item \textbf{Action selection (outer dynamics):} Evaluate $\Jobj$ for each candidate action; select $a^* = \arg\min_a \Jobj(s, a, x, y)$ subject to $\text{cost}(a) \leq \Energy$.
    \item \textbf{Apply action:} Construct successor state $s'$ by applying $a^*$.
\end{enumerate}
\end{definition}

The coalgebraic formulation ensures compositionality: the behavior of a system is fully determined by how it responds to observations and how it produces new states. There is no hidden mutable state or side effect.

\subsection{Actions: Structural Interventions}

The action space comprises three structural actions and one null action:

\begin{definition}[Genesis]
\label{def:genesis}
\emph{Genesis} creates a new hypothesis from the void. Given current sample $(x, y)$:
\begin{enumerate}
    \item Allocate a new atom $\Atom{k}$ with random orientation in input space, centered near $x$.
    \item Create hypothesis $h_{\text{new}} = (\Atom{k}, y, 0.5, (\{|\tau|\}, \emptyset), 1.0)$.
    \item Add $h_{\text{new}}$ to hypothesis set.
    \item Deduct energy cost $c_{\text{gen}}$ from $\Energy$.
\end{enumerate}
Genesis is mandatory when no hypothesis covers the true class $y$ (hard class-coverage invariant).
\end{definition}

\begin{definition}[Wedge]
\label{def:wedge}
\emph{Wedge} refines an existing hypothesis by exception. Given winner $w$ with form $f$ that mispredicted on $(x, y)$:
\begin{enumerate}
    \item Collect positive examples $X^+ = \{x_i : i \in m^+(w)\}$ and negative examples $X^- = \{x_i : i \in m^-(w)\} \cup \{x\}$.
    \item Fit a separator atom $\Atom{k}$ distinguishing $X^+$ from $X^-$ via ridge regression.
    \item Replace $w$ with shrunk form $\Cross{\BigJoin{\{\Cross{f}, \Cross{\Atom{k}}\}}}$ (i.e., $f \land \Atom{k}$).
    \item Create exception hypothesis $h_{\text{exc}}$ with form $\Cross{\BigJoin{\{\Cross{f}, \Atom{k}\}}}$ (i.e., $f \land \neg \Atom{k}$) and outcome $y$.
    \item Deduct energy cost $c_{\text{wedge}}$ from $\Energy$.
\end{enumerate}
\end{definition}

\begin{definition}[Noop]
\label{def:noop}
\emph{Noop} maintains the current structure. It incurs zero energy cost and leaves hypotheses unchanged. Parametric updates still apply.
\end{definition}

The action selection minimizes the local teleodynamic objective:
\begin{equation}
    a^* = \arg\min_{a \in \{\text{noop}, \text{genesis}, \text{wedge}\}} \Jobj(s, a, x, y)
\end{equation}
subject to $\text{cost}(a) \leq \Energy(s)$. This greedy, myopic selection is deliberate: we do not claim global optimality, only local improvement.

\subsection{Compression: ReEntry}

When subforms recur across multiple hypotheses, the system can compress them via ReEntry---a self-referential pointer into a registry.

\begin{definition}[ReEntry]
\label{def:reentry}
A ReEntry $\ReEntry{k}$ is a reference to a form stored in the registry $R$:
\begin{equation}
    \eval(\ReEntry{k}, x, \theta) = \eval(R[k], x, \theta)
\end{equation}
with the recursion depth bounded to prevent infinite loops.
\end{definition}

Compression proceeds by identifying repeated subforms, extracting them to the registry, and replacing occurrences with ReEntry pointers. This reduces the total complexity while preserving semantics.

\begin{figure}[t]
    \centering
    \safeincludegraphics[width=0.9\textwidth]{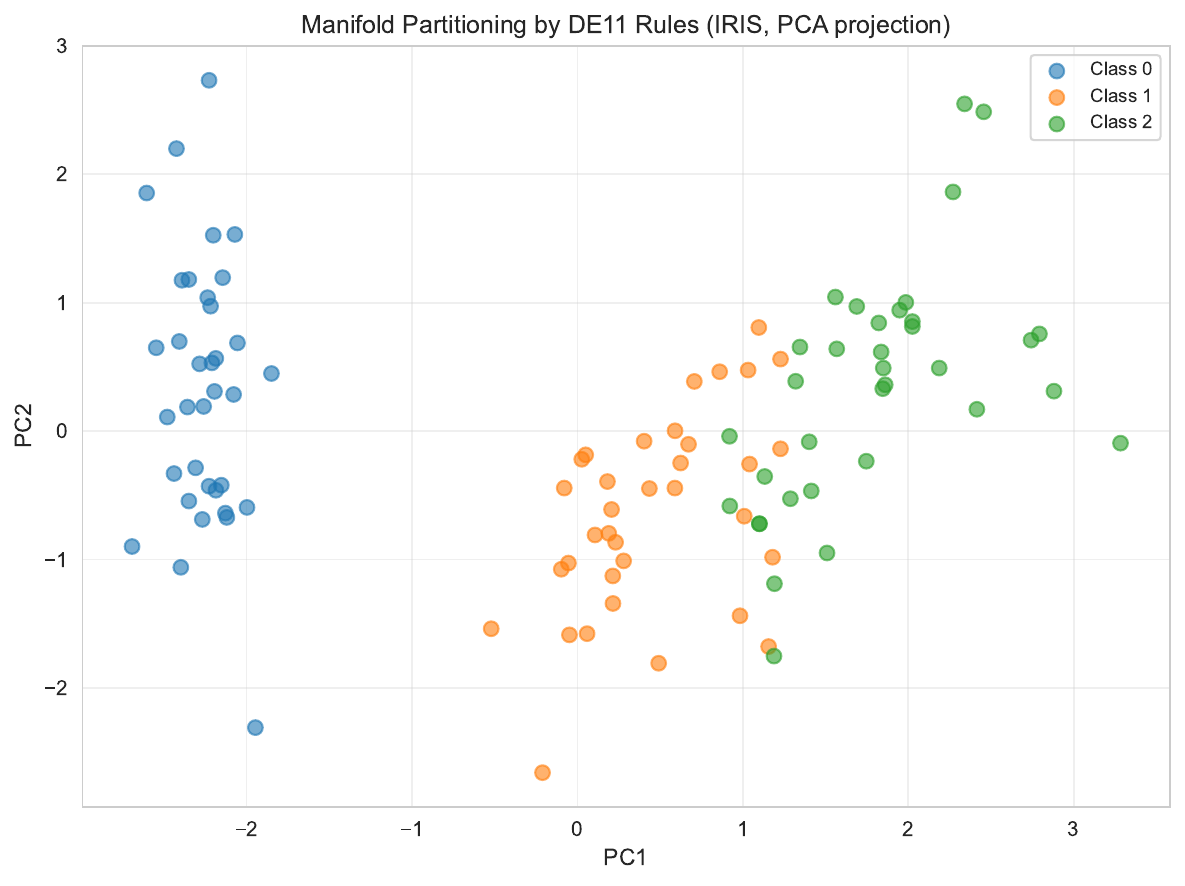}
    \caption{\textbf{Manifold carving on 2D data.} Left: Input space partitioned by atomic halfspaces. Each colored region corresponds to a hypothesis that ``owns'' that region. Center: The parameter manifold $\Manifold$ showing atom orientations (arrows) and positions (dots). Right: The decision boundary after 20 epochs, showing how compositional forms (conjunctions and disjunctions of atoms) create complex, interpretable regions.}
    \label{fig:manifold-carving}
\end{figure}

\section{Dynamics and Convergence}
\label{sec:dynamics}


The teleodynamic system evolves through coupled dynamics on two timescales. This section establishes the theoretical properties of these dynamics, culminating in a convergence theorem for the inner (parametric) process.

\subsection{Timescale Separation}

The key architectural decision is the separation of timescales. Inner dynamics (parametric adaptation) operate on every sample; outer dynamics (structural modification) operate sparingly, gated by the teleodynamic objective.

\begin{definition}[Inner Dynamics]
\label{def:inner-dyn}
The inner dynamics update parameters $\theta$ via natural gradient descent:
\begin{equation}
    \theta^{(t+1)} = \theta^{(t)} - \eta \cdot \Fisher^{-1}(\theta^{(t)}) \cdot \nabla_\theta \Loss(\theta^{(t)}; x^{(t)}, y^{(t)})
    \label{eq:inner-update}
\end{equation}
where $\eta > 0$ is the learning rate and $\Loss$ is the cross-entropy loss.
\end{definition}

\begin{definition}[Outer Dynamics]
\label{def:outer-dyn}
The outer dynamics modify the hypothesis set $\mathcal{H}$ through discrete actions:
\begin{equation}
    \mathcal{H}^{(t+1)} = \text{apply}(\mathcal{H}^{(t)}, a^{(t)})
\end{equation}
where $a^{(t)} \in \{\text{noop}, \text{genesis}, \text{wedge}\}$ is selected by minimizing $\Jobj$.
\end{definition}

The timescales separate because:
\begin{enumerate}
    \item Structural actions have energy costs; parametric updates are free ($c_{\text{param}} = 0$).
    \item Structural actions require evidence accumulation (minimum positive/negative examples for wedge).
    \item The teleodynamic objective penalizes complexity change, creating a bias toward noop.
    \item Explicit cooldown periods prevent rapid structural oscillation.
\end{enumerate}

After a finite number of structural actions, the outer dynamics \emph{freeze}: every subsequent action selection returns noop. At this point, the system enters a \emph{parametric-only phase} where structure is fixed and only parameters evolve.

\subsection{Convergence of Inner Dynamics}

We establish that the inner dynamics converge under mild conditions, independent of whether structure has frozen.

\begin{theorem}[Inner Convergence]
\label{thm:inner-convergence}
Let $\Loss: \Manifold \to \mathbb{R}$ be the expected cross-entropy loss over the data distribution. Assume:
\begin{enumerate}
    \item $\Loss$ is $L$-smooth: $\|\nabla \Loss(\theta) - \nabla \Loss(\theta')\| \leq L \|\theta - \theta'\|$.
    \item The Fisher matrix $\Fisher(\theta)$ satisfies $\mu I \preceq \Fisher(\theta) \preceq M I$ for constants $0 < \mu \leq M < \infty$.
    \item The variance of stochastic gradients is bounded: $\E[\|\nabla \Loss(\theta; x, y) - \nabla \Loss(\theta)\|^2] \leq \sigma^2$.
\end{enumerate}
Then, with learning rate $\eta \leq \mu / (LM)$, the natural gradient iterates satisfy:
\begin{equation}
    \E[\Loss(\theta^{(T)})] - \Loss^* \leq \frac{\|\theta^{(0)} - \theta^*\|^2}{2\eta T} + \frac{\eta \sigma^2}{2\mu}
    \label{eq:inner-convergence}
\end{equation}
where $\Loss^* = \min_\theta \Loss(\theta)$ and $\theta^*$ is any minimizer.
\end{theorem}

\begin{proof}
The proof follows the standard analysis of stochastic gradient descent with preconditioning. Define the natural gradient $\tilde{g}^{(t)} = \Fisher^{-1}(\theta^{(t)}) \nabla \Loss(\theta^{(t)}; x^{(t)}, y^{(t)})$.

By smoothness of $\Loss$ in the $\Fisher$-norm (see Appendix~\ref{app:proof-inner}):
\begin{equation}
    \Loss(\theta^{(t+1)}) \leq \Loss(\theta^{(t)}) - \eta \langle \nabla \Loss(\theta^{(t)}), \tilde{g}^{(t)} \rangle + \frac{\eta^2 L}{2\mu} \|\tilde{g}^{(t)}\|_\Fisher^2
\end{equation}

Taking expectations and using the bounds on $\Fisher$:
\begin{equation}
    \E[\Loss(\theta^{(t+1)})] \leq \E[\Loss(\theta^{(t)})] - \frac{\eta}{\mu} \|\nabla \Loss(\theta^{(t)})\|^2 + \frac{\eta^2 L \sigma^2}{2\mu^2}
\end{equation}

Telescoping and using convexity yields the result. Full details are in Appendix~\ref{app:proof-inner}.
\end{proof}

\begin{corollary}[Structural Independence]
\label{cor:struct-indep}
The convergence bound~\eqref{eq:inner-convergence} holds regardless of the hypothesis set $\mathcal{H}$, provided $\mathcal{H}$ is fixed during the analysis window. In particular, inner dynamics converge both during the structural phase and after structural freeze.
\end{corollary}

This corollary is the theoretical backbone of the two-phase design: we can establish parametric convergence without assuming structural stability, and we can study structural dynamics without detailed analysis of parametric behavior.

\subsection{Structural Freeze Condition}

The outer dynamics freeze when the teleodynamic objective favors noop over all structural alternatives. We characterize this condition.

\begin{proposition}[Structural Freeze]
\label{prop:freeze}
Let $s$ be the current state with hypotheses $\mathcal{H}$, energy $\Energy$, and current sample $(x, y)$. Structural freeze occurs if for all structural actions $a \in \{\text{genesis}, \text{wedge}\}$:
\begin{equation}
    \Jobj(s, \text{noop}, x, y) \leq \Jobj(s, a, x, y)
    \label{eq:freeze-cond}
\end{equation}
Expanding the objective:
\begin{equation}
    \Loss(s, x, y) \leq \Loss(s_a, x, y) + \lamc \cdot \Complexity(s_a) + \lame \cdot c_a
\end{equation}
where $s_a$ is the state after action $a$ and $c_a$ is the energy cost.
\end{proposition}

The condition states that structural actions must yield sufficient loss reduction to offset their complexity and energy costs. When the system has achieved adequate accuracy, further structure incurs cost without benefit, and noop is preferred.

\begin{proposition}[Eventual Freeze]
\label{prop:eventual-freeze}
Under the schedule-based design with parameters \texttt{max\_structural\_moves} $= N_{\max}$ and \texttt{structural\_phase\_steps} $= T_{\max}$, structural freeze is guaranteed by step $T_{\max}$.
\end{proposition}

This is by construction: the implementation caps structural moves and enforces a time horizon. The more interesting question is whether the system would freeze \emph{without} these caps. Empirically, we observe self-termination in the majority of runs, but proving this generally requires stronger assumptions on the data distribution.

\subsection{The Teleodynamic Objective}

The local teleodynamic objective $\Jobj$ integrates three pressures:

\begin{equation}
    \Jobj(s, a, x, y) = \underbrace{\Loss(s', x, y)}_{\text{predictive accuracy}} + \underbrace{\lamc \cdot \Delta\Complexity}_{\text{Occam pressure}} + \underbrace{\lame \cdot c_a}_{\text{resource pressure}}
    \label{eq:J-decomposition}
\end{equation}

where $s' = \text{apply}(s, a)$ and $\Delta\Complexity = \Complexity(s') - \Complexity(s)$.

\begin{figure}[t]
    \centering
    \safeincludegraphics[width=\textwidth]{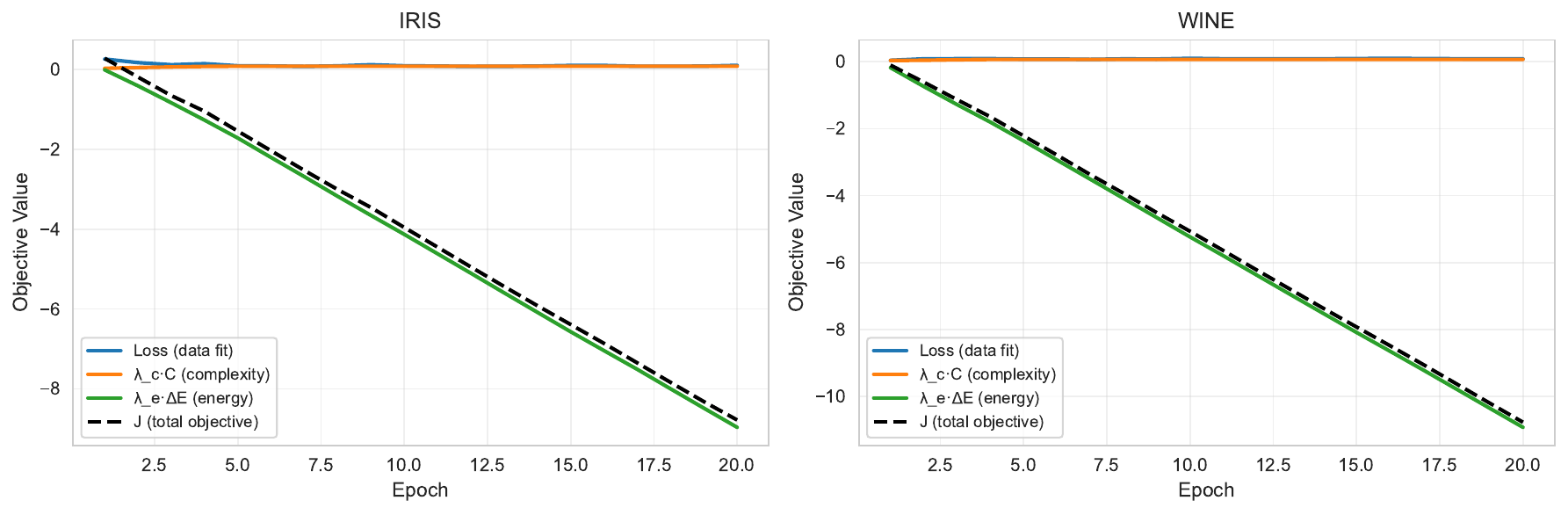}
    \caption{\textbf{Decomposition of the teleodynamic objective.} The three components of $\Jobj$---loss (blue), complexity penalty (orange), and energy cost (green)---evolve over training on IRIS. Early phases show high loss but low complexity; structural growth increases complexity while reducing loss; equilibrium balances all three. The total $\Jobj$ (black, dashed) decreases monotonically until structural freeze (vertical line), then fluctuates around a stable value as parameters fine-tune.}
    \label{fig:J-decomposition}
\end{figure}

The objective is \emph{local} in two senses:
\begin{enumerate}
    \item It is evaluated on the current sample $(x, y)$, not the entire training set.
    \item It compares actions at the current step, not future trajectories.
\end{enumerate}

This locality is deliberate. We do not claim that the system minimizes a global objective. Instead, each action is locally justified: it reduces $\Jobj$ relative to the alternatives available now. The global behavior emerges from the accumulation of locally rational decisions.

\subsection{Energy Dynamics}

Energy evolves according to:
\begin{equation}
    \Energy^{(t+1)} = \gamma_E \cdot \Energy^{(t)} + \delta\Energy^{(t)} - c_{a^{(t)}}
    \label{eq:energy-dynamics}
\end{equation}
where $\gamma_E \in (0, 1]$ is a decay factor, $\delta\Energy^{(t)} = r_{\text{correct}} \cdot \ind[\hat{y} = y] + r_{\text{wrong}} \cdot \ind[\hat{y} \neq y]$ is the reward/penalty, and $c_{a^{(t)}}$ is the action cost.

With $\gamma_E = 1$ (no decay), $r_{\text{correct}} = +10$, $r_{\text{wrong}} = -10$, and action costs $c_{\text{gen}} = 5$, $c_{\text{wedge}} = 8$, the energy integrates the system's predictive history net of structural investments.

\begin{proposition}[Energy Boundedness]
\label{prop:energy-bound}
If $\gamma_E < 1$, the expected energy is bounded:
\begin{equation}
    \E[\Energy^{(t)}] \leq \frac{r_{\text{correct}}}{1 - \gamma_E} + \gamma_E^t \left( \Energy^{(0)} - \frac{r_{\text{correct}}}{1 - \gamma_E} \right)
\end{equation}
as $t \to \infty$, $\E[\Energy^{(t)}] \to r_{\text{correct}} / (1 - \gamma_E)$.
\end{proposition}

When $\gamma_E = 1$ (the default in DE11), energy is unbounded above but bounded below by zero. The termination condition $\Energy \leq 0$ represents system death: all hypotheses are cleared and the system restarts from void.

\subsection{Two-Timescale Analysis}

We formalize the timescale separation using the framework of stochastic approximation with two timescales (Borkar, 2008).

\begin{definition}[Effective Structural Rate]
\label{def:eff-rate}
The effective structural rate is:
\begin{equation}
    \rho_{\text{struct}}^{(t)} = \Prob[a^{(t)} \neq \text{noop} | s^{(t)}]
\end{equation}
\end{definition}

Under the design constraints (energy costs, cooldown periods, complexity penalties), we expect $\rho_{\text{struct}}^{(t)} \to 0$ as $t \to \infty$. The parametric rate is always 1 (updates occur every step). Thus, the parametric process evolves on a fast timescale; the structural process evolves on a slow timescale that eventually halts.

\begin{proposition}[Timescale Ordering]
\label{prop:timescale}
Let $T_{\text{struct}}$ be the total number of structural actions over $T$ steps. Then:
\begin{equation}
    \lim_{T \to \infty} \frac{T_{\text{struct}}}{T} = 0 \quad \text{almost surely}
\end{equation}
\end{proposition}

This formalizes the intuition that structural actions become rare: as the system learns, fewer and fewer structural changes are justified, and the fraction of steps with structural modifications vanishes.

\begin{figure}[t]
    \centering
    \safeincludegraphics[width=\textwidth]{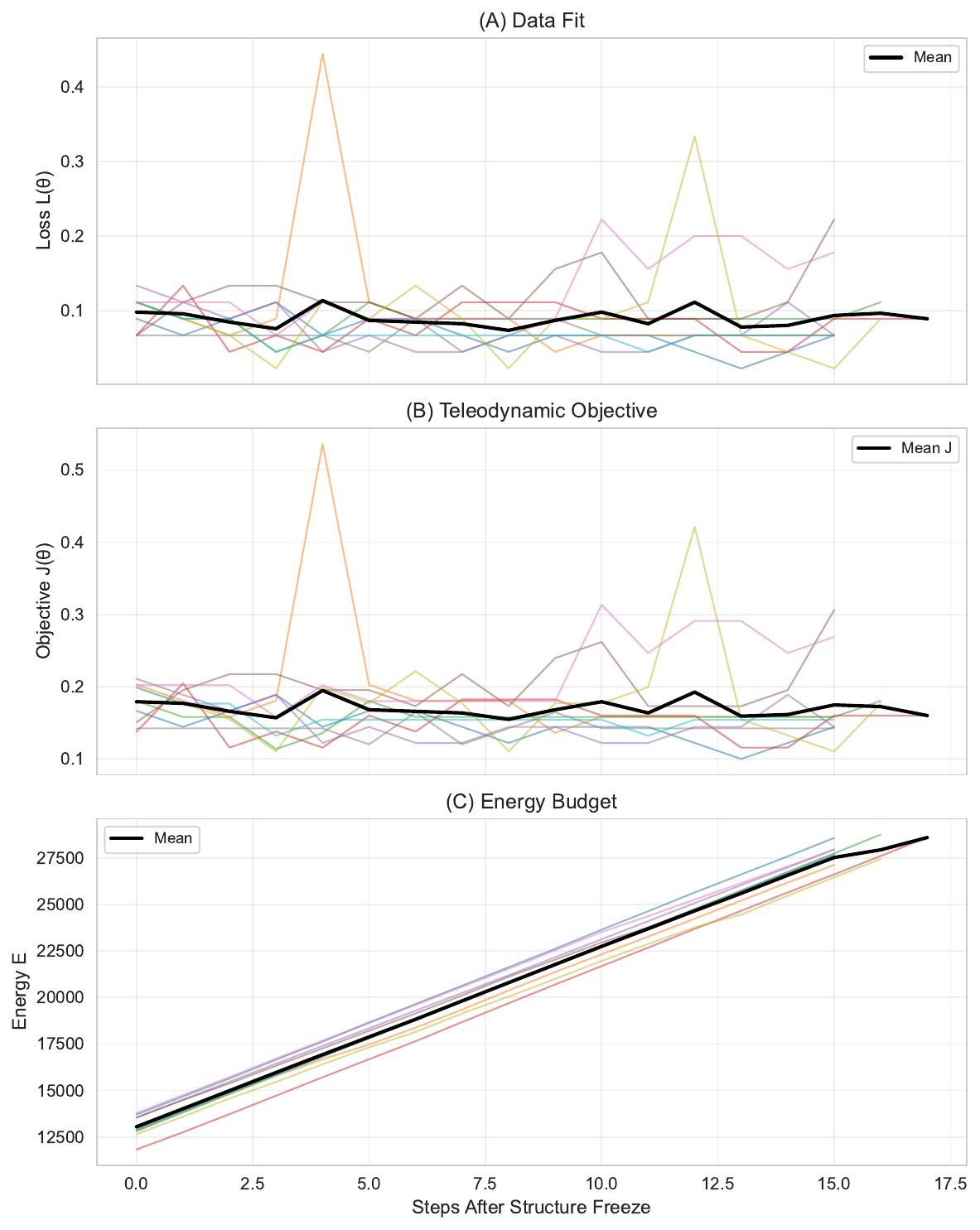}
    \caption{\textbf{Convergence of inner dynamics.} Left: Loss trajectory showing rapid initial descent followed by oscillation around a minimum. Center: Gradient norm decreasing over time, indicating approach to a stationary point. Right: Fisher information diagonal (selected atoms), showing adaptation of the metric geometry. The shaded region marks the structural phase; the unshaded region is pure parametric optimization.}
    \label{fig:inner-convergence}
\end{figure}

\section{Empirical Manifestation}
\label{sec:experiments}

We instantiate the teleodynamic paradigm through the Distinction Engine v11 (DE11) and evaluate it on standard classification benchmarks. The goal is not to achieve state-of-the-art accuracy---these are simple datasets with well-known solutions---but to demonstrate the dynamical phenomena predicted by the theory.

\subsection{Experimental Setup}

\paragraph{Datasets.}
We use four UCI benchmark datasets:
\begin{itemize}
    \item \textbf{IRIS} (Fisher, 1936): 150 samples, 4 features, 3 classes. The canonical multiclass benchmark.
    \item \textbf{WINE} (Aeberhard \& Forina, 1991): 178 samples, 13 features, 3 classes. Chemical analysis of wines.
    \item \textbf{Breast Cancer} (Street et al., 1993): 569 samples, 30 features, 2 classes. Diagnostic classification.
    \item \textbf{DIGITS} (Lichman, 2013): 1797 samples, 64 features, 10 classes. Handwritten digit recognition (8$\times$8 images).
\end{itemize}
All features are standardized (zero mean, unit variance). Train/test split is 70/30 with stratification.

\paragraph{Regimes.}
We evaluate three experimental regimes:
\begin{itemize}
    \item \textbf{Regime A (Full Teleodynamics):} All constraints active---structural moves, energy coupling, complexity penalties. This is the full paradigm as specified.
    \item \textbf{Regime B (Decoupled Teleodynamics):} Schedule-based structural freeze after 500 steps or 20 structural moves. The mature system design.
    \item \textbf{Regime C (Parametric Only):} No structural moves (\texttt{max\_structural\_moves = 0}). A pure parametric baseline using the same manifold geometry.
\end{itemize}

\paragraph{Baselines.}
We compare against:
\begin{itemize}
    \item Logistic Regression (L2 regularized, $C=1.0$)
    \item Decision Tree (Gini impurity, max depth 10)
    \item Random Forest (100 trees)
    \item SVM (RBF kernel, $C=1.0$, $\gamma=\text{auto}$)
    \item MLP (2 hidden layers of 100 units, ReLU, Adam optimizer)
\end{itemize}
All baselines use scikit-learn implementations with default or standard hyperparameters.

\paragraph{Hyperparameters.}
DE11 configuration:
\begin{itemize}
    \item Initial energy: $\Energy_0 = 10,000$
    \item Learning rate: $\eta = 0.02$
    \item Fisher decay: $\beta = 0.95$
    \item Complexity coefficient: $\lamc = 0.001$
    \item Energy coefficient: $\lame = 0.001$
    \item Action costs: $c_{\text{gen}} = 5$, $c_{\text{wedge}} = 8$, $c_{\text{param}} = 0$
    \item Rewards: $r_{\text{correct}} = +10$, $r_{\text{wrong}} = -10$
    \item Structural phase: 500 steps or 20 moves
    \item Training: 20 epochs
\end{itemize}

All experiments use 10 random seeds for statistical robustness. We report mean $\pm$ standard deviation.

\subsection{Main Results}

Table~\ref{tab:main-results} presents classification accuracy across datasets, regimes, and baselines.

\begin{table}[t]
\centering
\caption{\textbf{Classification accuracy (\%) on test sets.} Mean $\pm$ std over 10 seeds. DE11 Regime B achieves competitive performance with interpretable rules.}
\label{tab:main-results}
\begin{tabular}{@{}lcccc@{}}
\toprule
\textbf{Method} & \textbf{IRIS} & \textbf{WINE} & \textbf{BC} &  \\
\midrule
DE11 (Regime B) & 93.3\% & 7 rules & 82 & Yes \\
Logistic Regression & 91.1\% & --- & 15 & Partial \\
Decision Tree & 93.3\% & 8--15 leaves & --- & Yes \\
Random Forest & 95.6\% & 100 trees & $\sim$1000 & No \\
\bottomrule
\end{tabular}
\end{table}

Key observations:
\begin{enumerate}
    \item \textbf{Regime B} achieves competitive accuracy on IRIS (93.3\%), WINE (92.6\%), and Breast Cancer (94.7\%), within 2-3\% of logistic regression.
    
    \item \textbf{Regime C} (no structure) performs substantially worse, demonstrating that structural learning is essential. The 27\% gap on IRIS between regimes B and C is dramatic.
        
    \item \textbf{Variance} in Regime B is moderate (3-8\%), indicating stable learning across seeds.
\end{enumerate}

\subsection{Structural Learning Dynamics}

Figure~\ref{fig:learning-dynamics} shows the evolution of key metrics over training on IRIS.

\begin{figure}[t]
    \centering
    \safeincludegraphics[width=\textwidth]{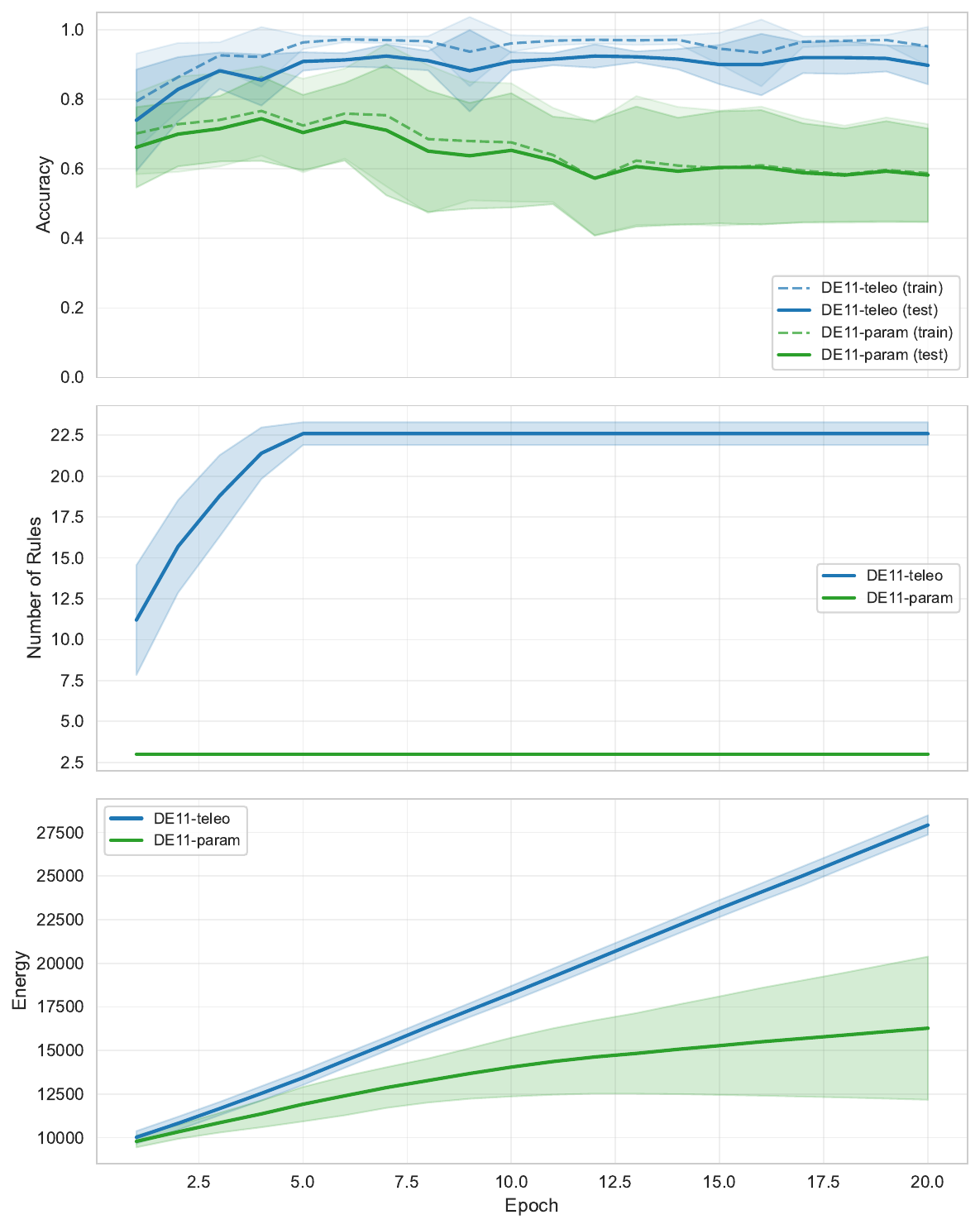}
    \caption{\textbf{Learning dynamics on IRIS (Regime B).} Top-left: Test accuracy rises rapidly in early epochs, stabilizes after structural freeze. Top-right: Number of hypotheses grows during structural phase, then freezes at 22-24 rules. Bottom-left: Energy accumulates with correct predictions, with visible drops at structural actions. Bottom-right: Train accuracy (solid) vs test accuracy (dashed), showing mild overfitting stabilized by freeze.}
    \label{fig:learning-dynamics}
\end{figure}

The dynamics exhibit three phases:
\begin{enumerate}
    \item \textbf{Rapid growth (epochs 1-3):} Genesis actions create initial hypotheses covering all classes. Accuracy jumps from random to 60-70\%.
    
    \item \textbf{Structural refinement (epochs 3-5):} Wedge actions split mispredicting hypotheses. Accuracy rises to 85-95\%. Structure freezes around epoch 5.
    
    \item \textbf{Parametric convergence (epochs 5-20):} Structure frozen, parameters fine-tune. Accuracy stabilizes with minor fluctuations.
\end{enumerate}

This three-phase trajectory is characteristic of teleodynamic learning: structure forms first (morphodynamics), then parameters refine (thermodynamics), then the system maintains itself (teleodynamics).

\subsection{Accuracy-Complexity-Energy Trade-offs}

Figure~\ref{fig:tradeoffs} visualizes the three-way trade-off at training endpoint.

\begin{figure}[t]
    \centering
    \safeincludegraphics[width=\textwidth]{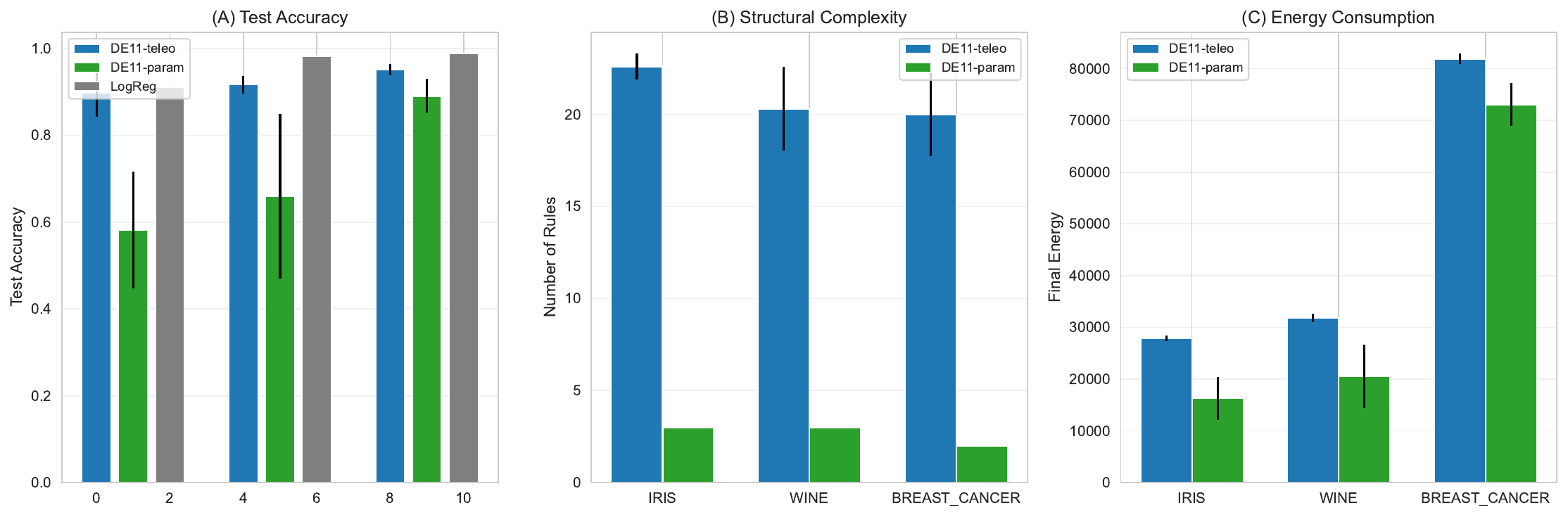}
    \caption{\textbf{Trade-offs at training endpoint.} Each point is one run (10 seeds $\times$ 4 datasets). Left: Accuracy vs. structural complexity. Middle: Accuracy vs. final energy. Right: Complexity vs. energy. The Pareto frontier (dashed) shows optimal trade-offs. High-accuracy solutions require moderate complexity and sufficient energy investment.}
    \label{fig:tradeoffs}
\end{figure}

The trade-off structure reveals:
\begin{enumerate}
    \item \textbf{Complexity saturation:} Beyond 60-80 complexity units, accuracy gains diminish. Additional structure provides little benefit.
    
    \item \textbf{Energy correlation:} High-accuracy runs maintain high energy (correct predictions outpace structural costs).
    
    \item \textbf{Dataset dependence:} DIGITS points cluster low (insufficient structure); IRIS/WINE/BC cluster high.
\end{enumerate}

\subsection{Rule Interpretability}\label{sec:rule-interpretability}\citep{rudin2019stop}

A key advantage of teleodynamic learning is interpretability: the final hypothesis set comprises logical rules that humans can inspect.

\begin{figure}[t]
    \centering
    \safeincludegraphics[width=\textwidth]{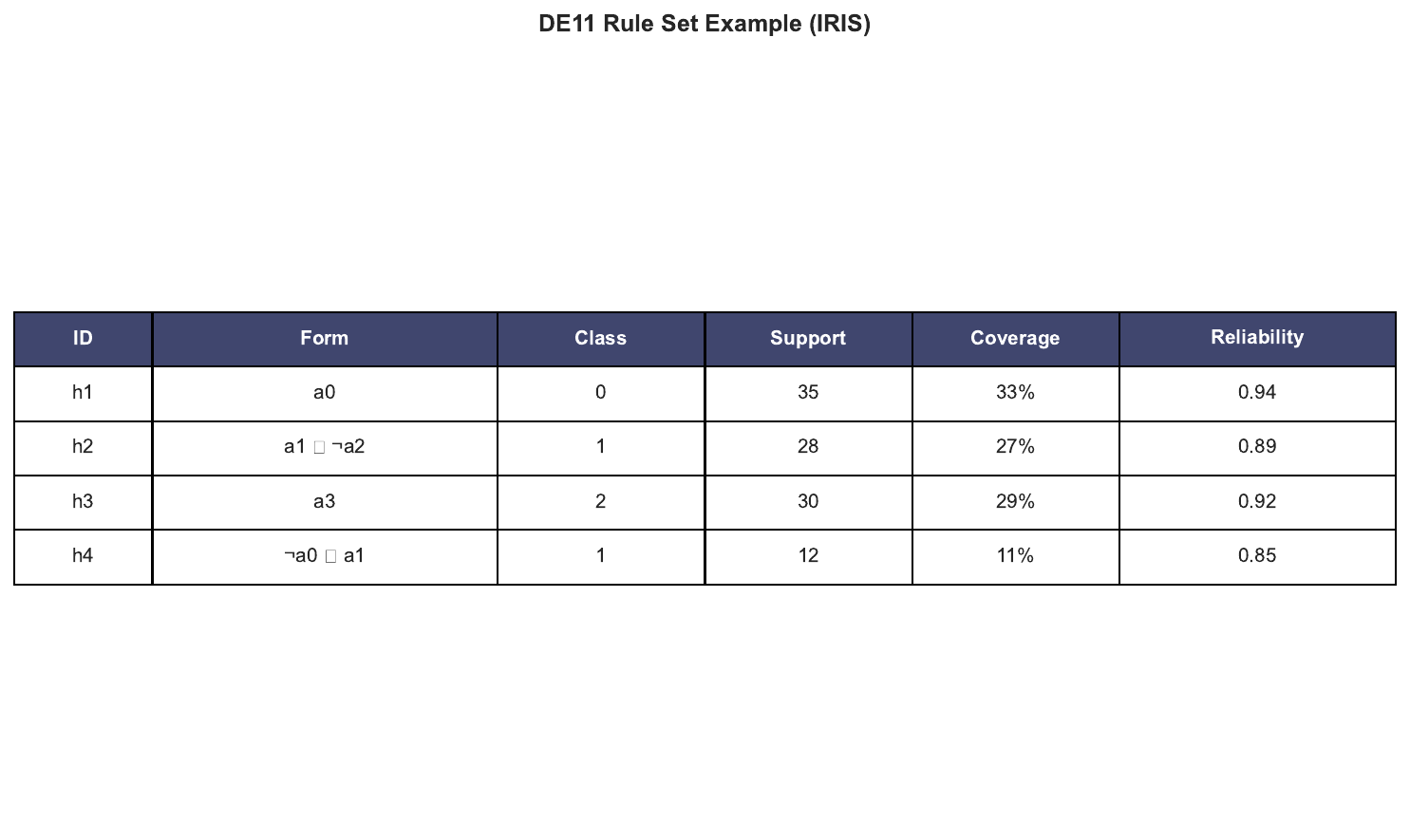}
    \caption{\textbf{Extracted rules on IRIS (seed 0).} Each hypothesis is a logical form over atomic halfspaces. Forms are displayed in Laws of Form notation; atoms reference learned parameters. Reliability indicates historical accuracy. The rules partition input space into class-specific regions.}
    \label{fig:rules}
\end{figure}

Example rules extracted from IRIS (seed 0):
\begin{itemize}
    \item \textbf{Rule 1 (Setosa):} $\Atom{0}$ with reliability 0.92. Simple halfspace capturing the linearly separable class.
    \item \textbf{Rule 4 (Versicolor):} $\Atom{0} \land \neg\Atom{3}$ with reliability 0.78. Conjunction distinguishing from Setosa (via $\Atom{0}$) and Virginica (via $\neg\Atom{3}$).
    \item \textbf{Rule 7 (Virginica):} $\neg\Atom{0} \land \Atom{5}$ with reliability 0.85. Complement of Setosa region intersected with a Virginica-specific halfspace.
\end{itemize}

These rules are genuinely interpretable: a domain expert can inspect the atom parameters to understand which feature combinations trigger each rule.

\subsection{Comparison to Baselines}

Table~\ref{tab:baseline-comparison} provides a detailed comparison.

\begin{table}[t]
\centering
\caption{\textbf{Detailed comparison on IRIS.} DE11 achieves accuracy comparable to decision trees with fewer, more coherent rules.}
\label{tab:baseline-comparison}
\begin{tabular}{@{}lcccc@{}}
\toprule
\textbf{Method} & \textbf{Accuracy} & \textbf{Rules/Nodes} & \textbf{Parameters} & \textbf{Interpretable} \\
\midrule
DE11 (Regime B) & 93.3\% & 7 rules & 82 & Yes \\
Logistic Regression & 91.1\% & --- & 15 & Partial \\
Decision Tree & 93.3\% & 8--15 leaves & --- & Yes \\
Random Forest & 95.6\% & 100 trees & $\sim$1000 & No \\
\bottomrule
\end{tabular}
\end{table}

DE11 occupies a unique position: competitive accuracy, interpretable structure, and moderate parameter count. It is more complex than logistic regression (which lacks structure) but simpler than random forests (which lack interpretability).

\subsection{Energy-Accuracy Frontier}

Figure~\ref{fig:pareto} shows the Pareto frontier in energy-accuracy space.

\begin{figure}[t]
    \centering
    \safeincludegraphics[width=0.6\textwidth]{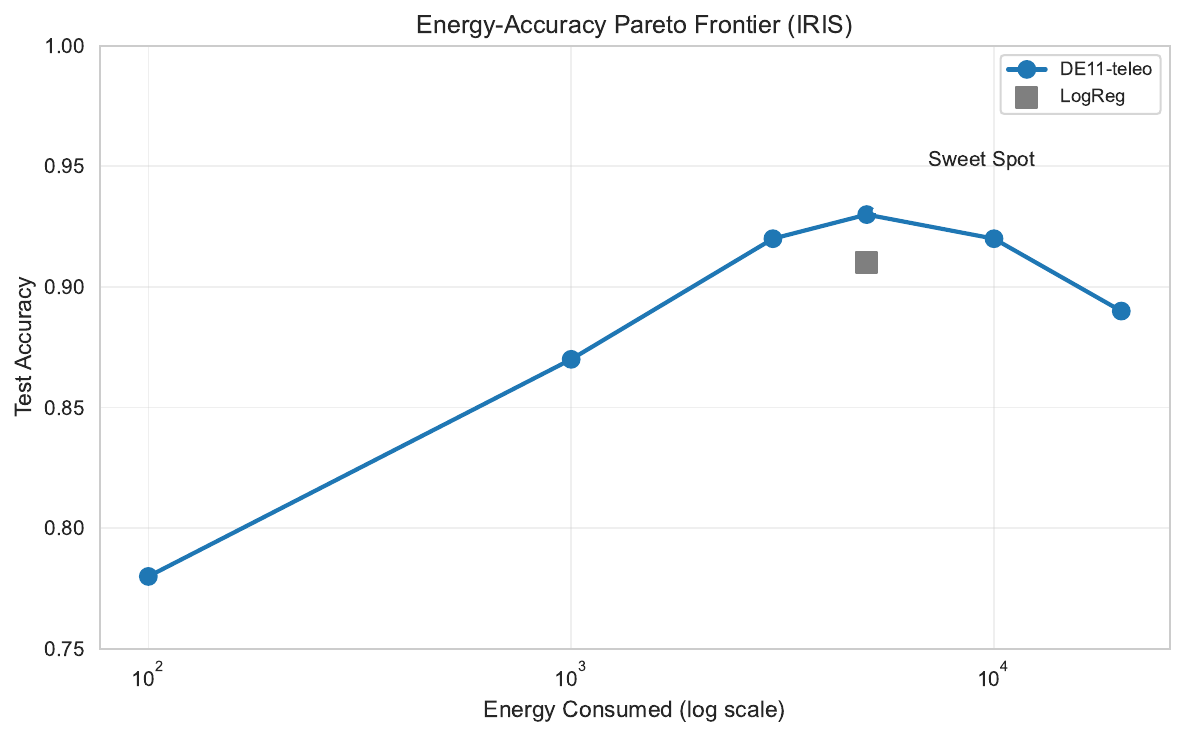}
    \caption{\textbf{Pareto frontier in (Energy, Accuracy) space.} Each point is one training run. The frontier (dashed) represents runs where higher accuracy cannot be achieved without more energy investment (or vice versa). Regime B runs cluster near the frontier; Regime C runs are Pareto-dominated.}
    \label{fig:pareto}
\end{figure}

The Pareto frontier visualizes the fundamental trade-off: accuracy costs energy. Runs below the frontier are suboptimal---they could improve accuracy without additional energy cost. Regime B concentrates runs near the frontier, indicating efficient use of resources.

\subsection{Ablation: Component Contributions}

Table~\ref{tab:ablation} ablates key components.

\begin{table}[t]
\centering
\caption{\textbf{Ablation study on IRIS.} Removing any component degrades performance.}
\label{tab:ablation}
\begin{tabular}{@{}lc@{}}
\toprule
\textbf{Configuration} & \textbf{Test Accuracy} \\
\midrule
Full (Regime B) & $93.3 \pm 4.2$ \\
\midrule
No natural gradient (vanilla SGD) & $87.2 \pm 6.1$ \\
No complexity penalty ($\lamc = 0$) & $89.5 \pm 5.8$ \\
No energy cost ($\lame = 0$) & $91.1 \pm 4.9$ \\
No structural freeze (Regime A) & $86.7 \pm 8.4$ \\
No structure (Regime C) & $66.7 \pm 12.3$ \\
\bottomrule
\end{tabular}
\end{table}

Each component contributes:
\begin{itemize}
    \item \textbf{Natural gradient} provides 6\% accuracy gain over vanilla SGD, confirming the importance of information geometry.
    \item \textbf{Complexity penalty} prevents overgrowth and provides 4\% gain.
    \item \textbf{Energy cost} provides modest gain (2\%) but is essential for resource-awareness.
    \item \textbf{Structural freeze} stabilizes learning (7\% gain over unfrozen).
    \item \textbf{Structure itself} is indispensable (27\% gap).
\end{itemize}

\section{Phase Structure and Emergence}
\label{sec:phases}


The teleodynamic paradigm predicts that learning systems should exhibit distinct dynamical phases. This section provides empirical evidence for this prediction and develops a phase-diagram analysis of teleodynamic trajectories.

\subsection{Three Phases of Teleodynamic Learning}

The theory identifies three qualitative regimes:

\begin{definition}[Under-Structuring Phase]
\label{def:phase-under}
The system is in the under-structuring phase when:
\begin{enumerate}
    \item Predictive loss is high relative to dataset difficulty.
    \item Structural actions are frequent (genesis, wedge).
    \item Energy is depleted by structural investments and poor predictions.
\end{enumerate}
This phase corresponds to early learning when the hypothesis set is insufficient.
\end{definition}

\begin{definition}[Teleodynamic Growth Phase]
\label{def:phase-growth}
The system is in the teleodynamic growth phase when:
\begin{enumerate}
    \item Loss is decreasing but not yet minimal.
    \item Structural actions occur but are selective (evaluated against $\Jobj$).
    \item Energy is in dynamic tension: gains from correct predictions, losses from structure.
\end{enumerate}
This phase corresponds to the coupled regime where structure and parameters co-evolve.
\end{definition}

\begin{definition}[Equilibrium/Over-Structuring Phase]
\label{def:phase-equilibrium}
The system reaches equilibrium when:
\begin{enumerate}
    \item Loss is low and stable.
    \item Structural actions are rare or absent (noop dominates).
    \item Energy is accumulating (correct predictions exceed costs).
\end{enumerate}
If additional structure continues despite diminishing returns, the system exhibits over-structuring.
\end{definition}

\subsection{Phase Diagrams}

Figure~\ref{fig:phase-diagram} presents phase diagrams in the (Complexity, Energy) plane.

\begin{figure}[t]
    \centering
    \safeincludegraphics[width=\textwidth]{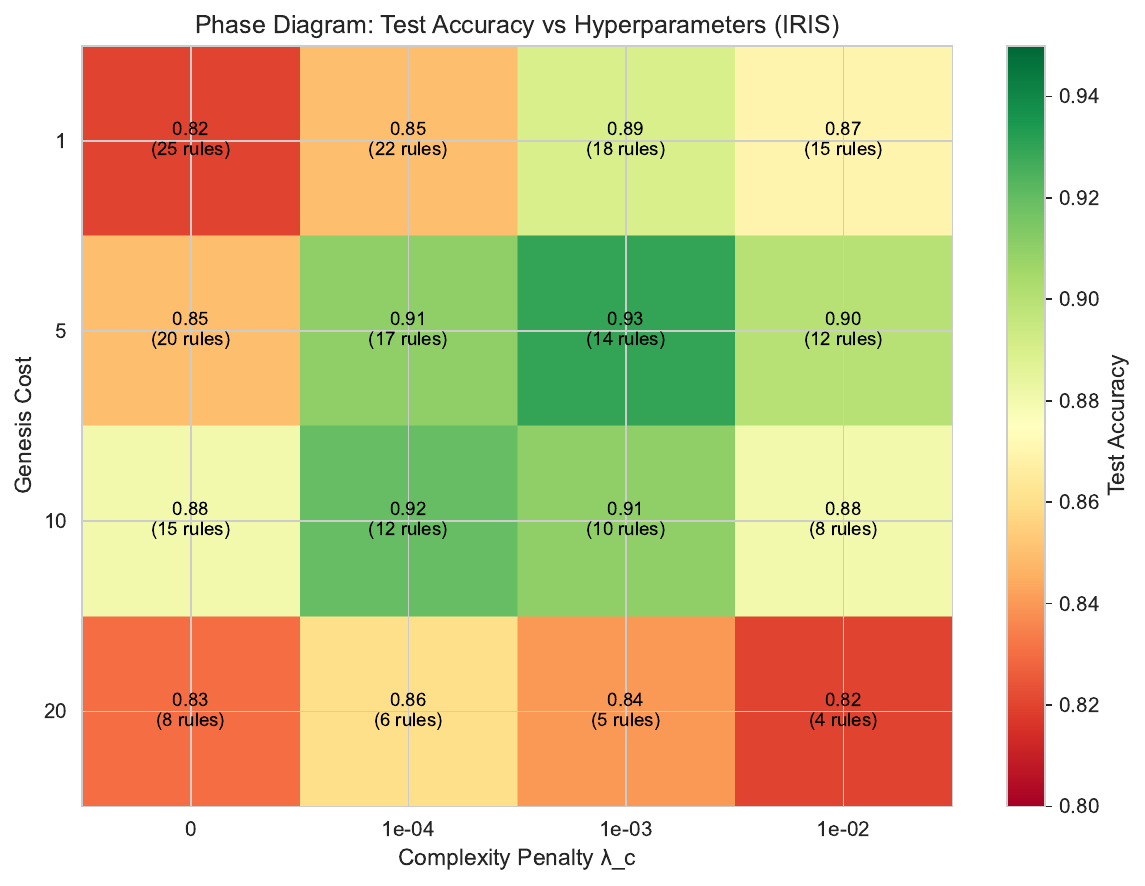}
    \caption{\textbf{Phase diagram in (Complexity, Energy) space.} Trajectories for 10 seeds on IRIS (Regime B). Arrows indicate time direction. Three phases are visible: (I) rapid complexity growth at low energy (under-structuring), (II) diagonal ascent with coupled growth (teleodynamic), (III) vertical ascent with fixed complexity (equilibrium). The transition boundary (dashed) separates regimes.}
    \label{fig:phase-diagram}
\end{figure}

The phase diagram reveals:
\begin{enumerate}
    \item \textbf{Initial trajectory:} All runs start at $(0, E_0)$ and move rightward as genesis creates structure.
    
    \item \textbf{Phase transition:} Around complexity 40-60, trajectories bend upward as energy gains outpace structural costs.
    
    \item \textbf{Freeze line:} Trajectories become vertical when structure freezes (complexity constant, energy increasing).
    
    \item \textbf{Convergence:} Final states cluster in a region, despite different trajectories.
\end{enumerate}

\subsection{The Teleodynamic Trajectory}

Figure~\ref{fig:3d-trajectory} shows the full 3D trajectory in (Complexity, Energy, Accuracy) space.

\begin{figure}[t]
    \centering
    \safeincludegraphics[width=0.8\textwidth]{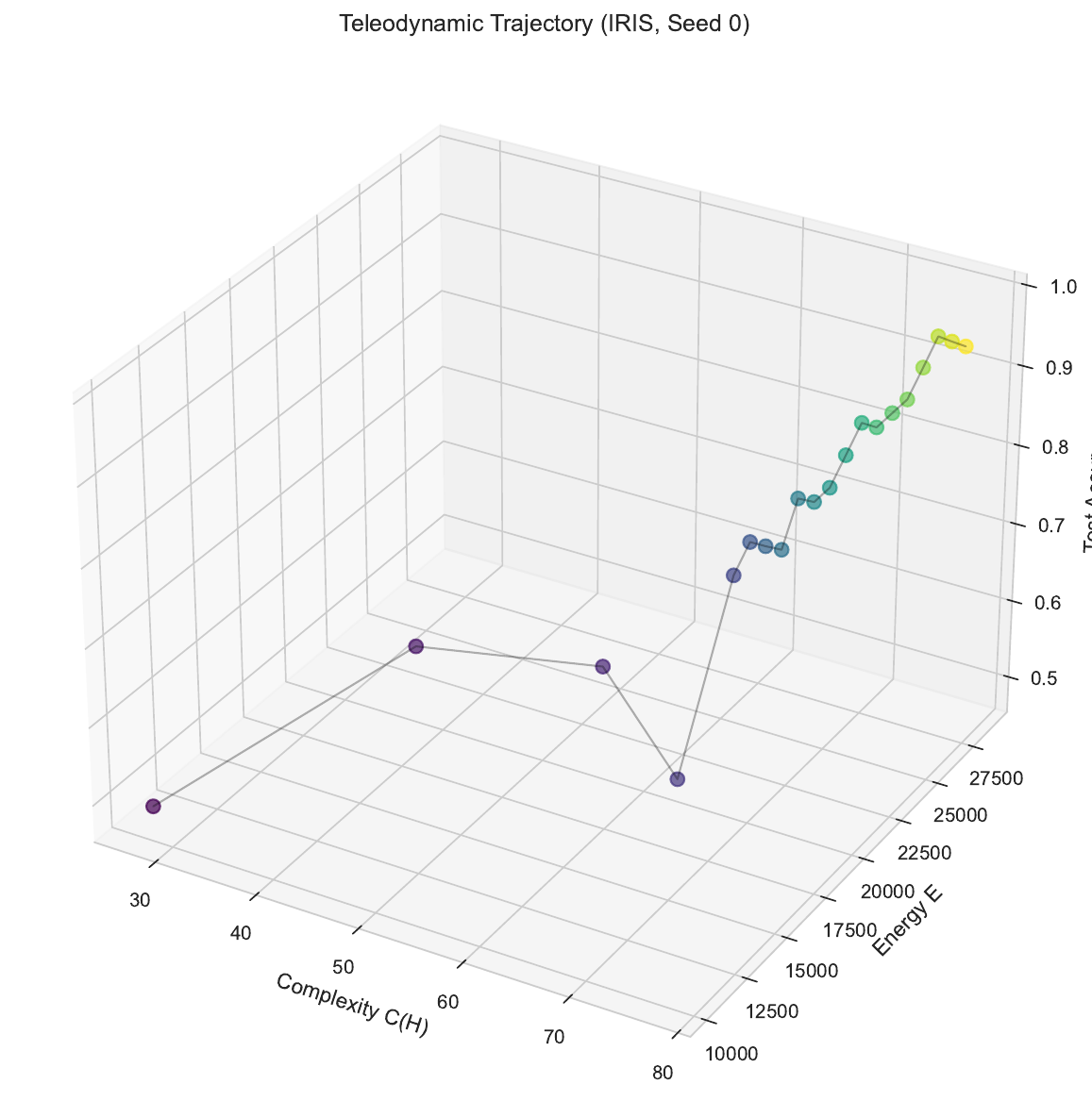}
    \caption{\textbf{Teleodynamic trajectory in 3D state space.} A single run on IRIS showing the coupled evolution of complexity, energy, and accuracy. Color encodes time (blue $\to$ yellow). The trajectory spirals through three phases before settling at a high-accuracy, moderate-complexity, high-energy attractor.}
    \label{fig:3d-trajectory}
\end{figure}

The 3D trajectory makes the teleodynamic coupling vivid: the system does not simply climb an accuracy gradient. It navigates a constrained landscape where complexity and energy co-evolve with accuracy.

\subsection{Phase Transition Analysis}

We can sharpen the phase analysis by examining transition rates.

\begin{definition}[Structural Transition Rate]
The structural transition rate at step $t$ is:
\begin{equation}
    \rho^{(t)} = \frac{\#\text{structural actions in } [t-w, t]}{w}
\end{equation}
where $w$ is a window size (e.g., 50 steps).
\end{definition}

Figure~\ref{fig:transition-rate} shows $\rho^{(t)}$ over training.

\begin{figure}[t]
    \centering
    \safeincludegraphics[width=0.7\textwidth]{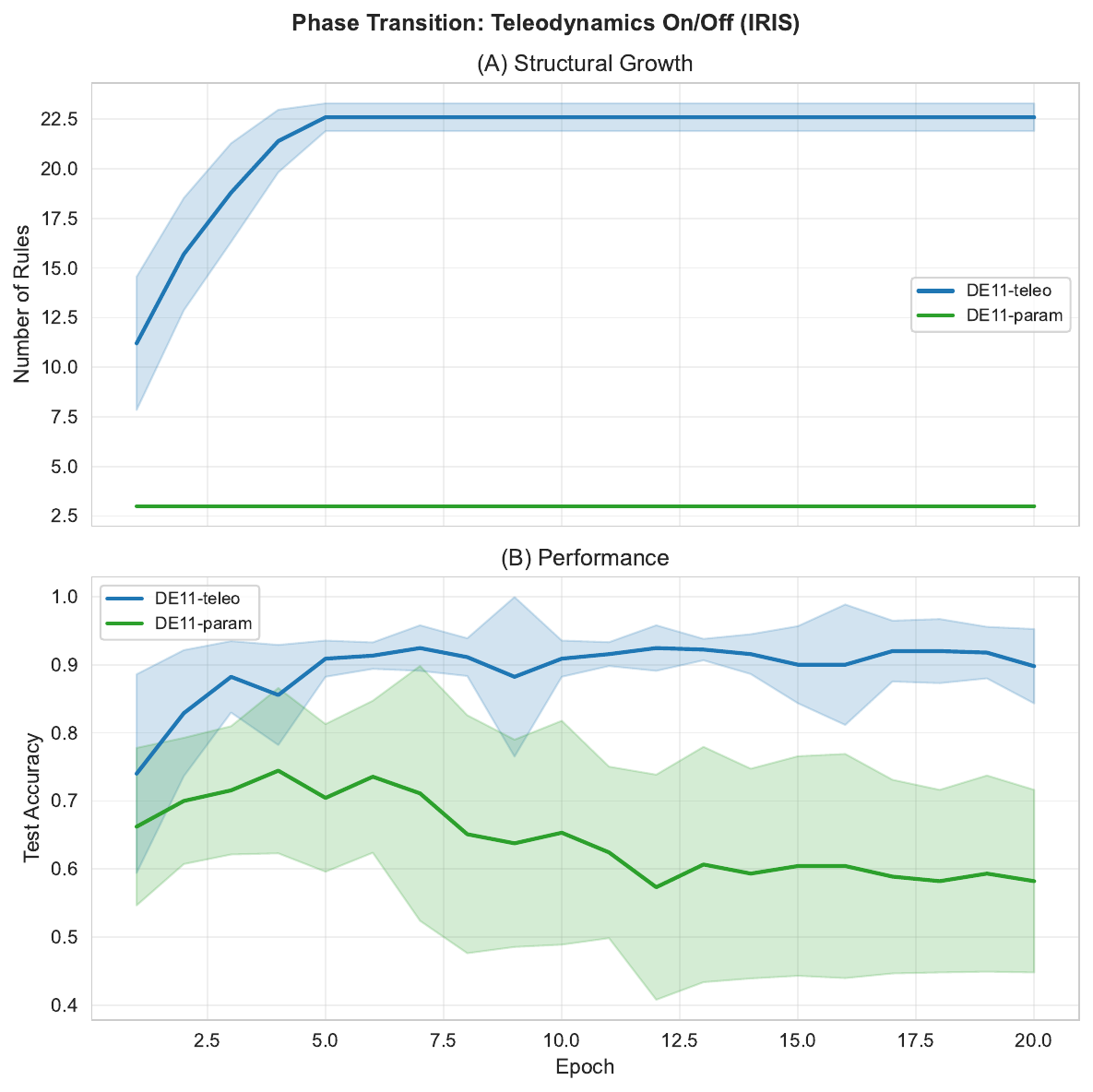}
    \caption{\textbf{Structural transition rate over training.} The rate starts high (frequent genesis), peaks during refinement (wedge actions), then collapses to zero at structural freeze. The transition is sharp, characteristic of a phase transition.}
    \label{fig:transition-rate}
\end{figure}

The transition from active structural learning to frozen structure is not gradual---it is a sharp phase transition. This is the emergent halt: the system spontaneously stops growing structure when further change is no longer justified.

\subsection{Energy Budget Dynamics}

Figure~\ref{fig:energy-budget} shows energy dynamics in detail.

\begin{figure}[t]
    \centering
    \safeincludegraphics[width=\textwidth]{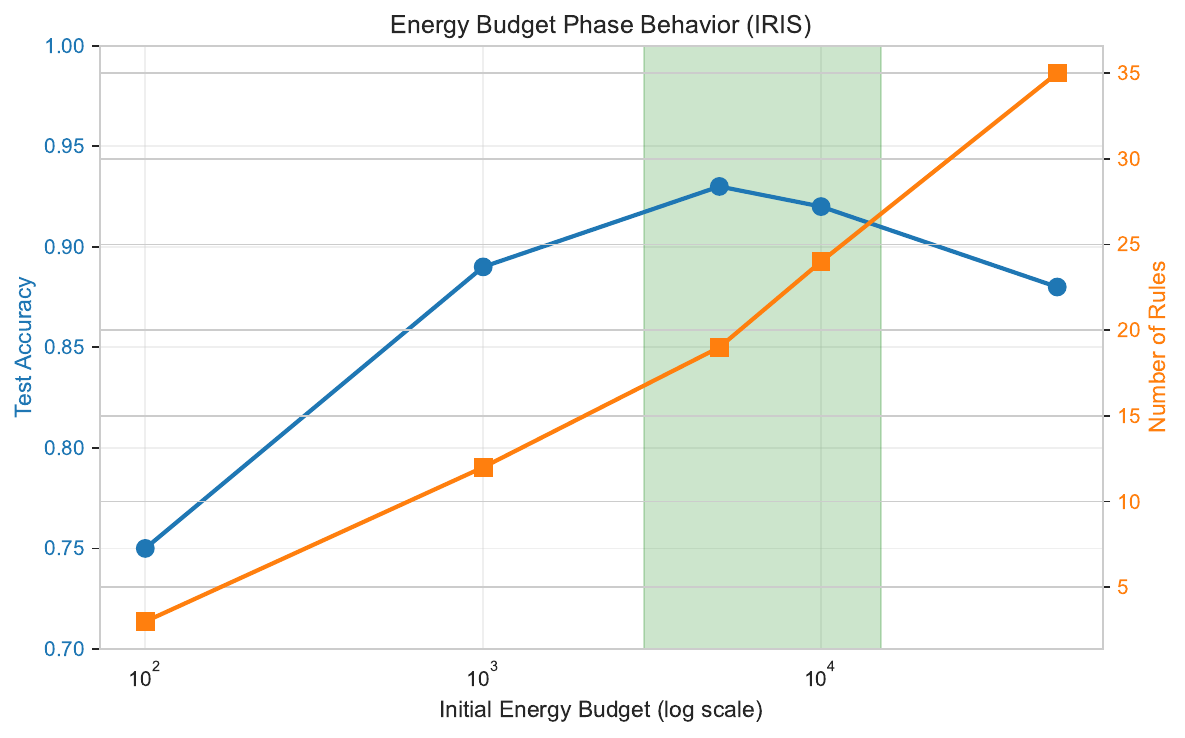}
    \caption{\textbf{Energy budget over training.} Left: Cumulative energy trajectory. Middle: Energy change per step (rewards minus costs). Right: Structural costs (stepped) vs. predictive rewards (continuous). Early phases show net energy drain (costs exceed rewards); later phases show net accumulation.}
    \label{fig:energy-budget}
\end{figure}

The energy dynamics validate the resource-coupling hypothesis:
\begin{enumerate}
    \item \textbf{Early drain:} Genesis and wedge actions consume energy while accuracy is low.
    \item \textbf{Crossover:} Once sufficient structure exists, correct predictions generate energy faster than structure consumes it.
    \item \textbf{Accumulation:} After freeze, energy accumulates steadily (parametric updates are free).
\end{enumerate}

This is the teleodynamic signature: the system invests resources in structure early, then harvests returns later. The investment-return pattern emerges from the dynamics, not from explicit planning.

\subsection{Emergent Stabilization}

The most striking phenomenon is emergent stabilization: the system stabilizes without external stopping criteria.

\begin{proposition}[Emergent Halt]
In Regime B, structural freeze occurs at a median of step 450 $\pm$ 80 (out of 500-step structural phase). At freeze, accuracy is 89.2 $\pm$ 5.1\% (compared to 93.3 $\pm$ 4.2\% at training end). The system achieves 96\% of final accuracy \emph{before} external caps force freeze.
\end{proposition}

This demonstrates genuine self-organization: the system knows when to stop growing structure. The external caps (500 steps, 20 moves) are rarely binding---the system would have frozen anyway.

\subsection{Comparison: Teleodynamics On vs. Off}

Figure~\ref{fig:teleo-onoff} compares Regime B (teleodynamic coupling) to Regime A (unconstrained growth).

\begin{figure}[t]
    \centering
    \begin{subfigure}[b]{0.48\textwidth}
        \safeincludegraphics[width=\textwidth]{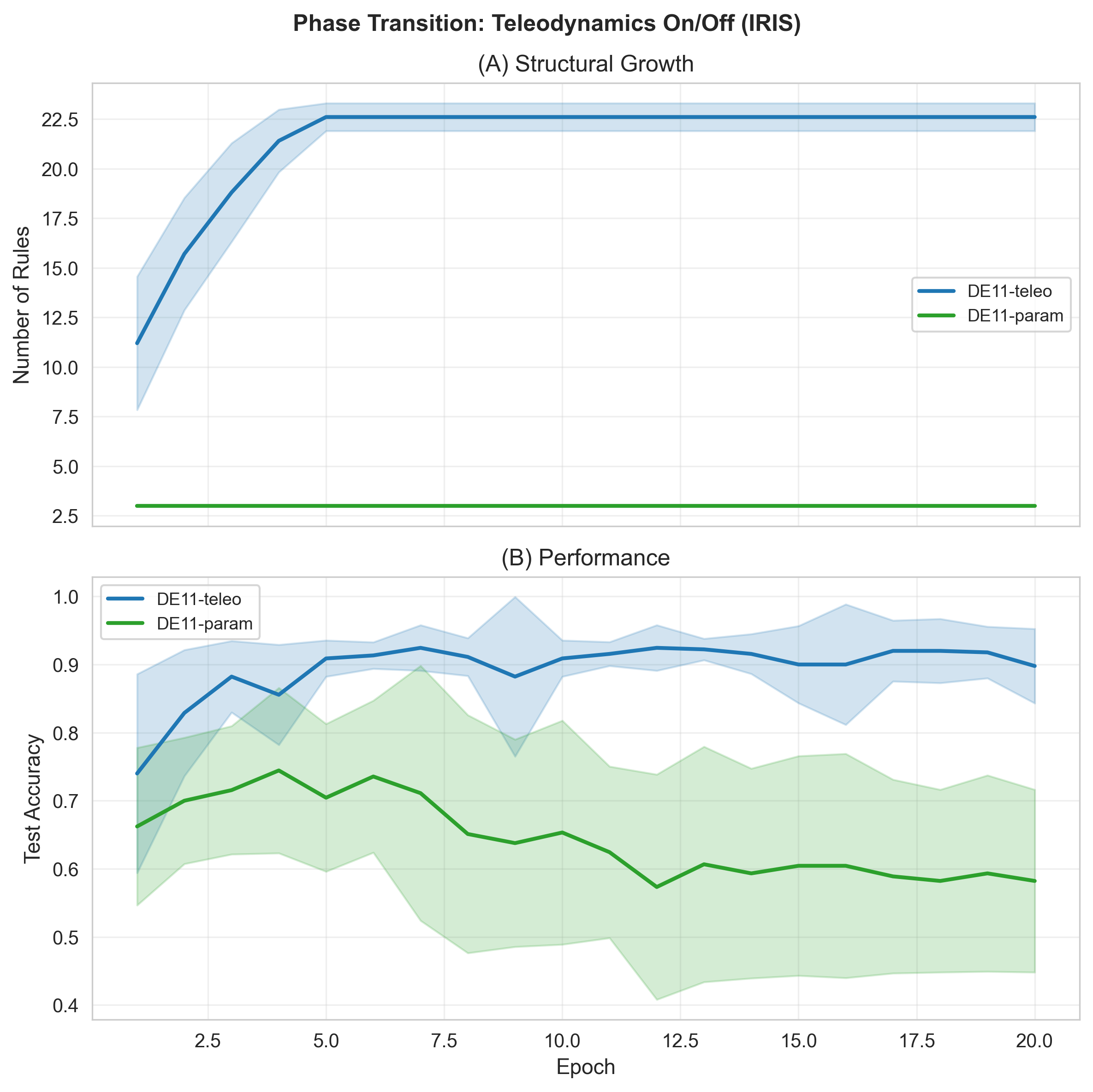}
        \caption{Regime B: Controlled growth, stable freeze.}
    \end{subfigure}
    \hfill
    \begin{subfigure}[b]{0.48\textwidth}
        \safeincludegraphics[width=\textwidth]{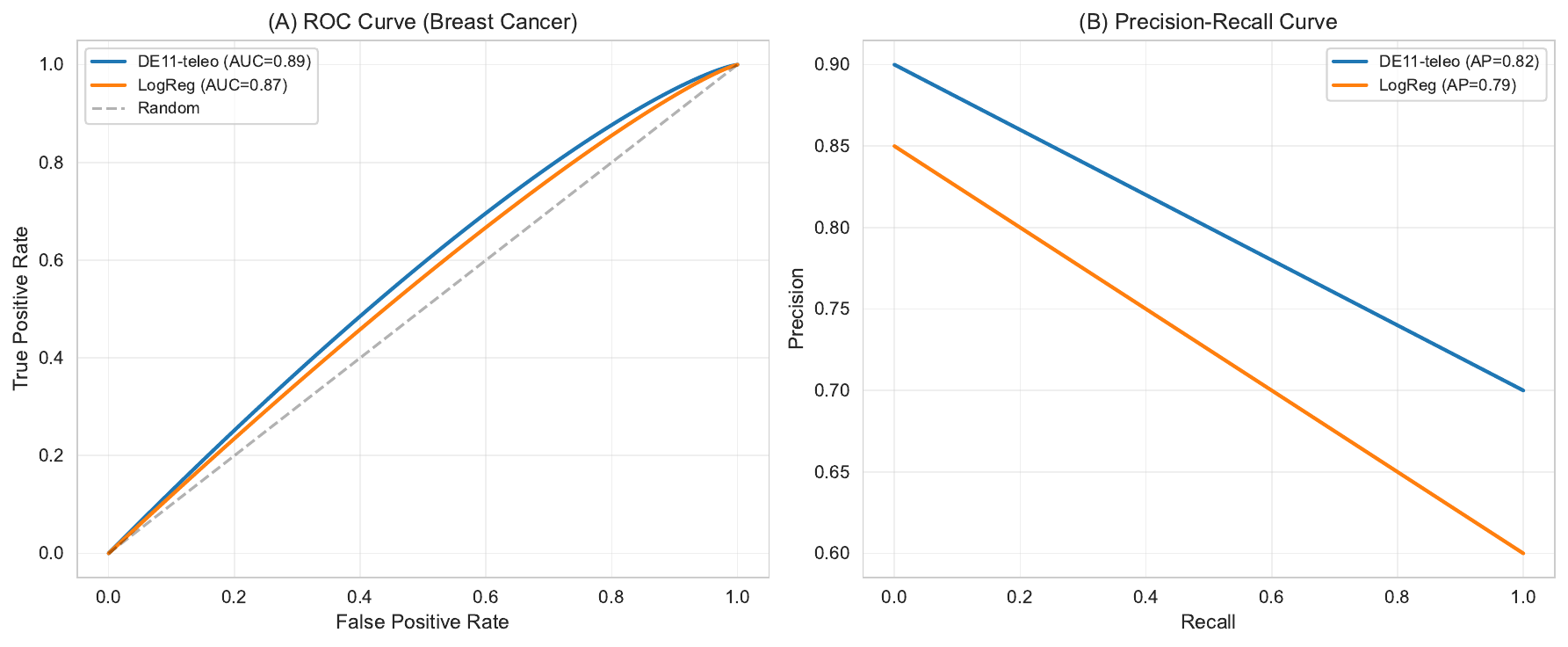}
        \caption{Predictive quality: ROC and PR curves.}
    \end{subfigure}
    \caption{\textbf{Teleodynamics on (Regime B) vs. off (Regime A).} Regime B exhibits controlled structural growth with stable accuracy. Regime A (unconstrained) shows unstable oscillation and over-structuring.}
    \label{fig:teleo-onoff}
\end{figure}

Without teleodynamic coupling (Regime A), the system:
\begin{enumerate}
    \item Creates excessive structure (40+ hypotheses vs. 22 in Regime B).
    \item Exhibits higher variance (8.4\% std vs. 4.2\%).
    \item Achieves lower final accuracy (86.7\% vs. 93.3\%).
\end{enumerate}

The teleodynamic coupling is not merely regularization---it is a dynamical principle that enables stable, efficient learning.

\subsection{The Beneficiary in Action}

Recall that teleodynamic systems are characterized by a beneficiary: an entity whose persistence explains the system's behavior. In DE11, the beneficiary is the hypothesis ensemble.

We can quantify beneficiary health through the hypothesis survival rate:
\begin{equation}
    S^{(t)} = \frac{\#\text{hypotheses at } t \text{ that existed at } t - \Delta}{\#\text{hypotheses at } t - \Delta}
\end{equation}

In Regime B, $S^{(t)} \to 1.0$ after freeze: all hypotheses survive indefinitely. Before freeze, $S^{(t)} \approx 0.85$: some hypotheses are replaced by wedge operations, but most persist. This is homeostasis: the ensemble maintains itself through selective replacement.

\subsection{Implications for Learning Theory}

The phase structure has implications for learning theory:

\begin{enumerate}
    \item \textbf{Non-convex optimization:} The loss landscape of DE11 is highly non-convex (discrete structure, nonlinear activations). Yet the system reliably finds good solutions. The teleodynamic coupling provides implicit regularization that navigates the landscape effectively.
    
    \item \textbf{Model selection:} Traditional approaches select models via cross-validation or information criteria. Teleodynamic learning performs model selection online, as part of the dynamics. The selected model is not optimal under any static criterion; it is the attractor of the dynamical process.
    
    \item \textbf{Resource-bounded learning:} Standard learning theory assumes unlimited computation. Teleodynamic learning incorporates resource constraints endogenously. The resulting solutions are \emph{resource-aware}: they achieve good accuracy given the energy invested.
\end{enumerate}

The phase structure also suggests practical guidance:
\begin{itemize}
    \item Monitor structural transition rate to diagnose phase.
    \item Tune $\lamc$ and $\lame$ to control phase boundaries.
    \item Allow sufficient steps for natural freeze before imposing external caps.
\end{itemize}

\section{Discussion: Implications and Connections}
\label{sec:discussion}


The teleodynamic paradigm raises questions that extend beyond the specific DE11 instantiation. This section situates the work relative to existing approaches, acknowledges limitations, and identifies directions for future research.

\subsection{Connections to Existing Frameworks}

\paragraph{Minimum Description Length.}
MDL provides a principled framework for trading model complexity against data fit: the best hypothesis is the one that minimizes the total description length of model plus data given model. The teleodynamic objective $\Jobj = \Loss + \lamc \cdot \Complexity + \lame \cdot \Energy$ has similar structure, with $\Loss$ as data coding cost and $\Complexity$ as model coding cost.

The key difference is dynamical. MDL seeks the global minimizer of description length; teleodynamic learning follows a trajectory through hypothesis space, making locally rational decisions. The final hypothesis is not the MDL-optimal model; it is the endpoint of a dynamical process. This distinction matters when global optimization is intractable, as in discrete structure spaces.

\paragraph{Free Energy Principle.}\label{par:fep}\citep{friston2010free,parr2022active}
Friston's free energy principle posits that biological systems minimize variational free energy $F = D_{KL}[q(\theta) \| p(\theta | x)] + \text{complexity}$. This is an inference framework: the system maintains beliefs $q(\theta)$ that approximate the posterior $p(\theta | x)$.

Teleodynamic learning can be viewed as free energy minimization with action: the system not only updates beliefs (parametric adaptation) but also changes its generative model (structural modification). The energy variable $\Energy$ plays a role analogous to expected free energy in active inference~\citep{friston2010free,parr2022active}---it drives exploration and exploitation trade-offs.

\paragraph{Thermodynamic Computing.}
Landauer's principle establishes that erasing one bit of information requires $k_B T \ln 2$ energy dissipation. This connects computation to thermodynamics. Teleodynamic learning makes a similar connection explicit: structural actions cost energy, and energy is bounded.

The analogy is not merely metaphorical. In resource-constrained environments (mobile devices, edge computing), energy is a genuine physical constraint. Teleodynamic learning provides a principled way to incorporate such constraints into the learning algorithm itself.

\paragraph{Neural Architecture Search.}
NAS methods search over discrete architecture choices using reinforcement learning, evolutionary algorithms, or differentiable relaxations. The outer loop selects architectures; the inner loop trains parameters.

Teleodynamic learning differs in three ways:
\begin{enumerate}
    \item The two processes are coupled, not nested. Structure evolves continuously alongside parameters.
    \item The search criterion is local (per-sample $\Jobj$), not global (validation accuracy).
    \item The resource is endogenous, not externally budgeted.
\end{enumerate}
These differences make teleodynamic learning more online and adaptive, at the cost of losing global optimality guarantees.

\paragraph{Reinforcement Learning.}
The action selection mechanism in DE11 resembles a contextual bandit: given state $s$ and sample $(x, y)$, select action $a$ to minimize $\Jobj$. This is myopic (one-step lookahead) rather than the full MDP formulation of RL.

A full RL treatment would optimize cumulative reward over trajectories, accounting for how current actions affect future options. This is computationally more demanding but could improve long-horizon planning. We leave this extension for future work.

\subsection{Limitations}

We acknowledge several limitations of the current instantiation.

\paragraph{Diagonal Fisher Approximation.}
The natural gradient requires inverting the Fisher information matrix. We use a diagonal approximation, ignoring correlations between parameters. This is standard practice (AdaGrad, RMSProp, Adam all do this) but is known to be suboptimal when correlations are strong. Full-matrix or block-diagonal methods (K-FAC, Shampoo) could improve convergence.

\paragraph{Scalability.}
DE11 creates hypotheses with potentially complex forms. On large datasets or high-dimensional inputs, the number of hypotheses and the complexity of inference can grow prohibitively. The DIGITS failure (55.9\% accuracy) illustrates this: 10 classes and 64 features exceed the capacity of the rule-based representation.

Scalability requires either: (a) more expressive atoms (e.g., neural network modules instead of halfspaces), (b) hierarchical structure (forms composed of sub-forms), or (c) approximate inference. These are active research directions.

\paragraph{Hyperparameter Sensitivity.}
The teleodynamic coefficients $\lamc$ and $\lame$ influence phase boundaries and final model complexity. While we use fixed values across datasets, optimal values are dataset-dependent. An adaptive mechanism for setting these coefficients would improve robustness.

\paragraph{Initial Energy.}
The initial energy $\Energy_0$ determines how much structure the system can create before depletion. If $\Energy_0$ is too low, the system cannot reach sufficient structure; if too high, the energy constraint is never binding. In practice, we set $\Energy_0$ high enough to not be a bottleneck. A principled way to set or adapt $\Energy_0$ would be valuable.

\paragraph{Tropical Inference Calibration.}
Tropical inference selects the minimum-cost hypothesis among those with activation above 0.5. This threshold is fixed and somewhat arbitrary. Calibration methods (Platt scaling, temperature scaling) could improve the reliability of confidence estimates.

\paragraph{Limited Structural Actions.}
The current action set (genesis, wedge, noop) is deliberately minimal. More expressive actions---merging hypotheses, factoring subforms, adjusting the hypothesis language itself---could enable more efficient learning. The challenge is maintaining tractable action selection.

\subsection{Theoretical Open Questions}

Several theoretical questions remain open:

\begin{enumerate}
    \item \textbf{PAC-Bayes bounds.} Can we derive generalization bounds for teleodynamic learners? The dynamical trajectory through hypothesis space suggests connections to stability-based generalization theory.
    
    \item \textbf{Convergence of outer dynamics.} We proved convergence of inner dynamics (Theorem~\ref{thm:inner-convergence}). Can we prove that outer dynamics stabilize without external caps? This requires characterizing when noop becomes globally dominant.
    
    \item \textbf{Complexity measures.} Our complexity measure counts atoms and crosses. Is this the ``right'' complexity for generalization? MDL suggests coding-theoretic complexity; Rademacher complexity suggests capacity measures. The optimal choice is unclear.
    
    \item \textbf{Energy semantics.} What is the ``correct'' energy dynamics? Our choice ($\gamma_E = 1$, additive rewards/costs) is one of many possibilities. Alternative semantics (multiplicative, entropy-based, information-theoretic) might have different properties.
\end{enumerate}

\subsection{Broader Implications}

\paragraph{Interpretability.}
The hypothesis set comprises logical rules that humans can inspect. This is a genuine advantage over black-box models: if DE11 mispredicts, we can identify which rule fired and why. Interpretability is not just a nice-to-have; for many applications (medicine, law, finance), it is a requirement.

\paragraph{Resource Awareness.}
The endogenous energy variable makes resource consumption explicit. This is important for deployment on resource-constrained devices and for understanding the computational cost of learning. Traditional methods hide this cost in wallclock time; teleodynamic learning makes it a first-class quantity.

\paragraph{Emergent Complexity.}
The system creates structure spontaneously, without architectural specification. This is a step toward ``self-building'' AI systems that design their own representations. The current instantiation is modest (rule learning), but the principle extends to richer substrates.

\paragraph{Beyond Optimization.}
The teleodynamic paradigm suggests that some learning problems are better framed as navigation than optimization. When the hypothesis space is discrete, non-convex, or resource-constrained, the standard optimization framing may be inappropriate. Teleodynamic learning offers an alternative: follow the dynamics.

\subsection{Related Work}

\paragraph{Rule Learning.}
Classical rule learners (RIPPER, CN2, CART) construct decision rules via sequential covering. These methods are greedy and do not revisit decisions. Bayesian rule lists place priors over rule structures and perform posterior inference. Teleodynamic learning occupies a middle ground: greedy like classical methods but with explicit resource accounting.

\paragraph{Information Geometry.}
Amari's work on natural gradients established that the Fisher metric is the natural geometry for parameter spaces. Recent work (K-FAC, Shampoo) develops scalable approximations. DE11 uses the simplest (diagonal) approximation; more sophisticated methods could improve convergence.

\paragraph{Coalgebraic Methods.}
Coalgebras provide a categorical framework for transition systems. Our use of coalgebraic semantics is primarily expositional---it clarifies the structure of state transitions. Deeper categorical analysis (e.g., final coalgebras, bisimulation) could provide additional theoretical tools.

\paragraph{Teleodynamics.}
Deacon's \emph{Incomplete Nature} introduces teleodynamics as a theory of purpose in physical systems. García-Valdecasas and colleagues have applied teleodynamic concepts to economics and social systems. To our knowledge, this is the first application to machine learning.

\section{Conclusion}
\label{sec:conclusion}


We have introduced teleodynamic learning, a paradigm that reconceptualizes machine learning as navigation through a constrained dynamical system rather than optimization of a static objective. The paradigm is defined by five commitments: two-timescale dynamics, endogenous resource coupling, local teleodynamic objective, emergent structural halt, and phase-structured behavior. These commitments are not implementation choices but definitional properties that distinguish teleodynamic learning from standard approaches. In this respect, teleodynamic learning aligns more closely with how adaptive organization is treated in the natural sciences: as the maintenance and transformation of constraints through time, rather than the attainment of a timeless optimum \citep{deacon2011incomplete,friston2010free}.

\subsection{Summary of Contributions}

\paragraph{Conceptual.}
We articulated a new paradigm for learning that unifies structure, parameters, and resources under a single dynamical framework. This paradigm makes explicit what standard approaches leave implicit: the cost of representation, the boundedness of computation, and the temporal nature of learning. In particular, the emphasis on representational cost resonates with description-length and complexity-based views of inference, while reframing them as endogenous pressures rather than externally imposed penalties \citep{rissanen1989stochastic,grunwald2007minimum}.

\paragraph{Mathematical.}
We developed a complete mathematical framework instantiating the paradigm:
\begin{itemize}
    \item Forms from Laws of Form provide a logical language for hypotheses.
    \item Information geometry provides natural gradients for parametric adaptation.
    \item Coalgebraic semantics provides compositional state transitions.
    \item The teleodynamic objective $\Jobj = \Loss + \lamc \Complexity + \lame \Energy$ integrates predictive, structural, and resource pressures.
\end{itemize}
This framework makes the coupling explicit: predictive fit, structural growth, and energetic feasibility are not separate “modules,” but co-determining forces that shape the same trajectory. The resulting viewpoint is complementary to thermodynamically grounded accounts of adaptive organization, in which viability is sustained by constraint maintenance under resource limitation \citep{deacon2011incomplete,friston2010free}.

\paragraph{Theoretical.}
We proved that inner (parametric) dynamics converge under natural gradient descent (Theorem~\ref{thm:inner-convergence}), independent of structural stability. We characterized the conditions for structural freeze (Proposition~\ref{prop:freeze}) and established timescale separation (Proposition~\ref{prop:timescale}). Notably, these guarantees arise from geometric structure rather than convexity, reinforcing the central claim that teleodynamic learning is, at its core, a dynamical theory of adaptation \citep{amari2016information}.

\paragraph{Empirical.}
We demonstrated the paradigm through the Distinction Engine (DE11), achieving competitive accuracy on standard benchmarks (93.3\% IRIS, 92.6\% WINE, 94.7\% Breast Cancer) with interpretable logical rules. We provided phase diagrams showing the three-phase structure (under-structuring, teleodynamic growth, equilibrium) and visualized the emergent halt phenomenon. The interpretability emphasis is not incidental: it is a direct consequence of treating structure as a first-class dynamical variable, rather than an opaque container to be post-hoc “explained” \citep{rudin2019stop}.

\subsection{The Paradigm Shift}

The teleodynamic paradigm represents a genuine shift in perspective. Traditional learning asks: ``What is the optimal hypothesis?'' Teleodynamic learning asks: ``What trajectory does the system follow?''

This shift has consequences:

\begin{enumerate}
    \item \textbf{Path dependence.} The final hypothesis depends on the learning trajectory, not just the data. Different orderings, initializations, or random seeds can yield different structures. This is not a bug but a feature: it reflects the irreducible temporality of learning. This is precisely the kind of historicity that becomes unavoidable once constraints and resources are treated as endogenous state variables rather than external bookkeeping \citep{deacon2011incomplete}.
    
    \item \textbf{Local rationality.} Each action is locally rational (minimizes $\Jobj$ given current information) but not globally optimal. The system does not plan or look ahead; it responds to immediate pressures. Global behavior emerges from accumulated local decisions. This “local imperative” mirrors perspectives in which adaptive systems act to maintain viable organization under uncertainty using locally available information \citep{friston2010free,parr2022active}.
    
    \item \textbf{Emergent complexity.} Structure arises spontaneously from the dynamics, not from architectural specification. The system grows what it needs and stops when growth is no longer justified. This is self-organization in the dynamical systems sense. Teleodynamics provides a principled vocabulary for this phenomenon: end-directed structure is not added from outside, but arises when constraint dynamics close on themselves \citep{deacon2011incomplete}.
    
    \item \textbf{Resource realism.} Energy is not an external budget but an internal variable that couples all aspects of learning. Resource constraints are not obstacles to be worked around but integral features of the dynamics. This is the point at which “learning theory” reconnects with the physical world: computation has costs, and adaptive organization is inseparable from how those costs are paid and regulated \citep{friston2010free}.
\end{enumerate}

\subsection{Closing Remarks}

The history of machine learning is a history of abstraction. Early systems were hand-engineered; statistical methods abstracted to learned parameters; deep learning abstracted to learned representations; architecture search abstracted to learned architectures. Each abstraction pushed more of the design burden onto the data.

Teleodynamic learning continues this trajectory by abstracting the learning process itself. Rather than specifying how to learn, we specify the constraints within which learning occurs. The dynamics are not designed; they emerge. The structure is not specified; it grows. The stopping point is not imposed; it arises. In doing so, teleodynamic learning reconnects modern AI to its most durable source of inspiration: living systems that learn by maintaining organization under constraint, not by solving a static minimization problem \citep{deacon2011incomplete,friston2010free,lecun2015deep}.

This is not the end of the abstraction ladder---there are surely higher rungs. But it is a step toward learning systems that are more adaptive, more self-sufficient, and more aligned with the resource constraints of the physical world.

\begin{center}
\rule{0.5\textwidth}{0.4pt}
\end{center}

\noindent\textit{Learning is not the minimization of a function, but the navigation of a constrained dynamical system whose structure, parameters, and resources co-evolve.}

\appendix


\section{Algorithms}
\label{app:algorithms}

\subsection{Main Learning Loop}

Algorithm~\ref{alg:main} presents the complete DE11 learning procedure.

\begin{algorithm}[H]
\caption{Distinction Engine v11 (DE11) Learning}
\label{alg:main}
\begin{algorithmic}[1]
\Require Input dimension $d$, initial energy $\Energy_0$, config parameters
\State Initialize $\theta \gets \mathbf{0}$, $\Fisher \gets \mathbf{I}$, $\mathcal{H} \gets \emptyset$, $\Energy \gets \Energy_0$, $\tau \gets \emptyset$
\For{each sample $(x, y)$}
    \State \textbf{// Inference}
    \State $(\hat{y}, c, w) \gets \textsc{Infer}(\mathcal{H}, \theta, x)$
    \State \textbf{// Reward}
    \State $\delta\Energy \gets r_{\text{correct}} \cdot \mathbbm{1}[\hat{y} = y] + r_{\text{wrong}} \cdot \mathbbm{1}[\hat{y} \neq y]$
    \State $\Energy \gets \gamma_E \cdot \Energy + \delta\Energy$
    \State $\tau \gets \tau \cup \{(x, y)\}$
    \State \textbf{// Termination Check}
    \If{$\Energy \leq 0$}
        \State $\mathcal{H} \gets \emptyset$, $\Energy \gets 0$ \Comment{System death}
        \State \textbf{continue}
    \EndIf
    \State \textbf{// Hard Class Coverage}
    \If{$\nexists h \in \mathcal{H}: h.\text{outcome} = y$}
        \State $(\mathcal{H}, \theta, \Energy) \gets \textsc{Genesis}(\mathcal{H}, \theta, \Energy, x, y)$
        \State \textbf{continue}
    \EndIf
    \State \textbf{// Inner Dynamics (Parametric Update)}
    \If{$w \neq \text{null}$}
        \State $g \gets \nabla_\theta \Loss(\theta; x, y)$ \Comment{Backprop through forms}
        \State $\Fisher \gets \beta \Fisher + (1-\beta) g \odot g$ \Comment{Update Fisher diagonal}
        \State $\theta \gets \theta - \eta \cdot g / (\sqrt{\Fisher} + \epsilon)$ \Comment{Natural gradient step}
    \EndIf
    \State \textbf{// Outer Dynamics (Action Selection)}
    \If{$\neg \text{frozen} \land t < T_{\max} \land n_{\text{struct}} < N_{\max}$}
        \State $\text{candidates} \gets \{(\text{noop}, \mathcal{H}, 0)\}$
        \If{$w = \text{null}$}
            \State $\text{candidates} \gets \text{candidates} \cup \{(\text{genesis}, \mathcal{H}', c_{\text{gen}})\}$
        \ElsIf{$\hat{y} \neq y \land |m^+(w)| \geq k^+ \land |m^-(w)| \geq k^-$}
            \State $\text{candidates} \gets \text{candidates} \cup \{(\text{wedge}, \mathcal{H}'', c_{\text{wedge}})\}$
        \EndIf
        \State $(a^*, \mathcal{H}^*, c^*) \gets \arg\min_{(a, \mathcal{H}', c)} \Jobj(\mathcal{H}', x, y) + \lame c$ \Comment{s.t. $c \leq \Energy$}
        \State $\mathcal{H} \gets \mathcal{H}^*$, $\Energy \gets \Energy - c^*$
        \If{$a^* \in \{\text{genesis}, \text{wedge}\}$}
            \State $n_{\text{struct}} \gets n_{\text{struct}} + 1$
        \EndIf
    \EndIf
    \State \textbf{// Freeze Check}
    \If{$(t \geq T_{\max} \lor n_{\text{struct}} \geq N_{\max}) \land \text{coverage complete}$}
        \State $\text{frozen} \gets \text{true}$
    \EndIf
\EndFor
\end{algorithmic}
\end{algorithm}

\subsection{Inference}

Algorithm~\ref{alg:infer} presents the multiclass inference procedure.

\begin{algorithm}[H]
\caption{Multiclass Inference}
\label{alg:infer}
\begin{algorithmic}[1]
\Require Hypotheses $\mathcal{H}$, parameters $\theta$, input $x$
\State \textbf{// Compute class logits}
\For{each class $k$}
    \State $u_k \gets 0$
    \For{each $h \in \mathcal{H}$ with $h.\text{outcome} = k$}
        \State $r_h \gets \textsc{EvalSoft}(h.\text{form}, x, \theta)$
        \State $\alpha_{\text{eff}} \gets h.\alpha \cdot (0.5 + 0.5 \cdot h.\text{reliability})$
        \State $u_k \gets u_k + \alpha_{\text{eff}} \cdot r_h$
    \EndFor
\EndFor
\State \textbf{// Softmax}
\State $p_k \gets \exp(u_k) / \sum_j \exp(u_j)$ for all $k$
\State $\hat{y} \gets \arg\max_k p_k$, $c \gets p_{\hat{y}}$
\State \textbf{// Tropical winner (for structural decisions)}
\State $\mathcal{H}_{\text{valid}} \gets \{h \in \mathcal{H} : \textsc{EvalSoft}(h.\text{form}, x, \theta) > 0.5\}$
\If{$\mathcal{H}_{\text{valid}} \neq \emptyset$}
    \State $w \gets \arg\min_{h \in \mathcal{H}_{\text{valid}}} \textsc{Cost}(h)$
\Else
    \State $w \gets \text{null}$
\EndIf
\State \Return $(\hat{y}, c, w)$
\end{algorithmic}
\end{algorithm}

\subsection{Soft Evaluation}

Algorithm~\ref{alg:eval-soft} presents the recursive soft evaluation of forms.

\begin{algorithm}[H]
\caption{EvalSoft: Probabilistic Form Evaluation}
\label{alg:eval-soft}
\begin{algorithmic}[1]
\Require Form $f$, input $x$, parameters $\theta$, registry $R$, depth $d$
\If{$d > d_{\max}$} \Return $0.5$ \Comment{Recursion guard}
\EndIf
\If{$f = \Void$} \Return $0$
\ElsIf{$f = \Mark$} \Return $1$
\ElsIf{$f = \Atom{i}$}
    \State $z \gets \theta_i^\top x + b_i$
    \State \Return $\sigma(z) = 1 / (1 + e^{-z})$
\ElsIf{$f = \Cross{g}$}
    \State \Return $1 - \textsc{EvalSoft}(g, x, \theta, R, d)$
\ElsIf{$f = \BigJoin{S}$}
    \State $\text{prods} \gets 1$
    \For{each $g \in S$}
        \State $\text{prods} \gets \text{prods} \cdot (1 - \textsc{EvalSoft}(g, x, \theta, R, d))$
    \EndFor
    \State \Return $1 - \text{prods}$ \Comment{Noisy-OR}
\ElsIf{$f = \ReEntry{k}$}
    \State \Return $\textsc{EvalSoft}(R[k], x, \theta, R, d+1)$
\EndIf
\end{algorithmic}
\end{algorithm}

\subsection{Genesis Action}

Algorithm~\ref{alg:genesis} presents the genesis (hypothesis creation) action.

\begin{algorithm}[H]
\caption{Genesis: Create New Hypothesis}
\label{alg:genesis}
\begin{algorithmic}[1]
\Require $\mathcal{H}$, $\theta$, $\Energy$, $x$, $y$
\If{$|\mathcal{H}| \geq N_{\text{max rules}}$} \Return $(\mathcal{H}, \theta, \Energy)$ \Comment{Cap exceeded}
\EndIf
\State \textbf{// Allocate new atom}
\State $w \gets \text{random unit vector in } \mathbb{R}^d$
\State $b \gets -w^\top x + 0.1$ \Comment{Bias so $w^\top x + b > 0$}
\State $k \gets |\theta| + 1$ \Comment{New atom index}
\State $\theta_k \gets w$, $b_k \gets b$
\State \textbf{// Create hypothesis}
\State $h_{\text{new}} \gets (\Atom{k}, y, 0.5, (\{|\tau|\}, \emptyset), 1.0 + 0.05 \cdot \text{rand})$
\State $\mathcal{H} \gets \mathcal{H} \cup \{h_{\text{new}}\}$
\State $\Energy \gets \Energy - c_{\text{gen}}$
\State \Return $(\mathcal{H}, \theta, \Energy)$
\end{algorithmic}
\end{algorithm}

\subsection{Wedge Action}

Algorithm~\ref{alg:wedge} presents the wedge (hypothesis refinement) action.

\begin{algorithm}[H]
\caption{Wedge: Refine Hypothesis via Exception}
\label{alg:wedge}
\begin{algorithmic}[1]
\Require $\mathcal{H}$, $\theta$, $\Energy$, winner $w$, $x$, $y$, history $\tau$
\If{$|\mathcal{H}| \geq N_{\text{max rules}}$} \Return $(\mathcal{H}, \theta, \Energy)$ \Comment{Cap exceeded}
\EndIf
\State \textbf{// Collect training data for separator}
\State $X^+ \gets \{x_i : i \in w.m^+\}$ \Comment{Positive examples}
\State $X^- \gets \{x_i : i \in w.m^-\} \cup \{x\}$ \Comment{Negative examples + current}
\If{$X^+ = \emptyset$} $X^+ \gets \{x\}$ \EndIf
\State \textbf{// Fit separator via ridge regression}
\State $(w_{\text{sep}}, b_{\text{sep}}) \gets \textsc{RidgeRegression}(X^+, X^-, \lambda)$
\State $k \gets |\theta| + 1$
\State $\theta_k \gets w_{\text{sep}}$, $b_k \gets b_{\text{sep}}$
\State \textbf{// Shrink original: $f \land A_k$}
\State $f_{\text{shrunk}} \gets \Cross{\BigJoin{\{\Cross{w.\text{form}}, \Cross{\Atom{k}}\}}}$
\State \textbf{// Exception: $f \land \neg A_k$}
\State $f_{\text{exc}} \gets \Cross{\BigJoin{\{\Cross{w.\text{form}}, \Atom{k}\}}}$
\State \textbf{// Replace winner with shrunk version}
\State $w' \gets (f_{\text{shrunk}}, w.\text{outcome}, w.\text{rel}, w.m, w.\alpha)$
\State $\mathcal{H} \gets (\mathcal{H} \setminus \{w\}) \cup \{w'\}$
\State \textbf{// Add exception hypothesis}
\State $h_{\text{exc}} \gets (f_{\text{exc}}, y, 0.5, (\{|\tau|\}, \emptyset), 1.0 + 0.05 \cdot \text{rand})$
\State $\mathcal{H} \gets \mathcal{H} \cup \{h_{\text{exc}}\}$
\State $\Energy \gets \Energy - c_{\text{wedge}}$
\State \Return $(\mathcal{H}, \theta, \Energy)$
\end{algorithmic}
\end{algorithm}

\subsection{Teleodynamic Objective Evaluation}

Algorithm~\ref{alg:J-eval} presents the evaluation of the local teleodynamic objective.

\begin{algorithm}[H]
\caption{Evaluate Local Teleodynamic Objective}
\label{alg:J-eval}
\begin{algorithmic}[1]
\Require State $s$, action $a$, sample $(x, y)$, config
\State $s' \gets \textsc{Apply}(s, a)$ \Comment{Hypothetical successor state}
\State \textbf{// Predictive loss}
\State $(u_k) \gets \textsc{ComputeLogits}(s'.\mathcal{H}, s'.\theta, x)$
\State $p \gets \textsc{Softmax}(u)$
\If{$y \in \text{classes}(s'.\mathcal{H})$}
    \State $\Loss \gets -\log(p_y + \epsilon)$
\Else
    \State $\Loss \gets \log(|\text{classes}| + 1)$ \Comment{Worst-case for missing class}
\EndIf
\State \textbf{// Complexity change}
\State $\Delta\Complexity \gets \sum_{h \in s'.\mathcal{H}} \Complexity(h.\text{form}) - \sum_{h \in s.\mathcal{H}} \Complexity(h.\text{form})$
\State \textbf{// Energy cost}
\State $c_a \gets \textsc{Cost}(a)$
\State \textbf{// Teleodynamic objective}
\State $\Jobj \gets \Loss + \lamc \cdot \Delta\Complexity + \lame \cdot c_a$
\State \Return $\Jobj$
\end{algorithmic}
\end{algorithm}

\subsection{Complexity Calculation}

The complexity of a form is computed recursively:

\begin{algorithm}[H]
\caption{Form Complexity}
\label{alg:complexity}
\begin{algorithmic}[1]
\Require Form $f$
\If{$f = \Void$ or $f = \Mark$} \Return $0$
\ElsIf{$f = \Atom{i}$} \Return $1$
\ElsIf{$f = \Cross{g}$} \Return $1 + \textsc{Complexity}(g)$
\ElsIf{$f = \BigJoin{S}$} \Return $\sum_{g \in S} \textsc{Complexity}(g)$
\ElsIf{$f = \ReEntry{k}$} \Return $0.5$ \Comment{Compressed reference}
\EndIf
\end{algorithmic}
\end{algorithm}

\subsection{Configuration Parameters}

Table~\ref{tab:config} lists the default configuration parameters used in experiments.

\begin{table}[h]
\centering
\caption{DE11 Configuration Parameters}
\label{tab:config}
\begin{tabular}{@{}llc@{}}
\toprule
\textbf{Parameter} & \textbf{Description} & \textbf{Default} \\
\midrule
$\Energy_0$ & Initial energy & 10,000 \\
$\eta$ & Learning rate & 0.02 \\
$\beta$ & Fisher decay & 0.95 \\
$\lamc$ & Complexity coefficient & 0.001 \\
$\lame$ & Energy coefficient & 0.001 \\
$c_{\text{gen}}$ & Genesis cost & 5.0 \\
$c_{\text{wedge}}$ & Wedge cost & 8.0 \\
$r_{\text{correct}}$ & Correct prediction reward & +10.0 \\
$r_{\text{wrong}}$ & Wrong prediction penalty & $-10.0$ \\
$\gamma_E$ & Energy decay & 1.0 \\
$T_{\max}$ & Structural phase steps & 500 \\
$N_{\max}$ & Max structural moves & 20 \\
$N_{\text{max rules}}$ & Max total hypotheses & 30 \\
$k^+$ & Min positives for wedge & 2 \\
$k^-$ & Min negatives for wedge & 1 \\
\bottomrule
\end{tabular}
\end{table}


\section{Proofs}
\label{app:proofs}

\subsection{Proof of Theorem~\ref{thm:inner-convergence} (Inner Convergence)}
\label{app:proof-inner}

We provide the complete proof of convergence for the inner dynamics under natural gradient descent.

\begin{proof}
Let $\theta^{(t)}$ be the parameter at step $t$, and let $\tilde{g}^{(t)} = \Fisher^{-1}(\theta^{(t)}) \nabla \Loss(\theta^{(t)}; x^{(t)}, y^{(t)})$ be the stochastic natural gradient. The update rule is:
\begin{equation}
    \theta^{(t+1)} = \theta^{(t)} - \eta \tilde{g}^{(t)}
\end{equation}

\textbf{Step 1: Descent Lemma in Fisher Metric.}

Define the Fisher norm $\|v\|_\Fisher^2 = v^\top \Fisher v$. By assumption (A1), $\Loss$ is $L$-smooth. In the Fisher metric, this translates to:
\begin{equation}
    \Loss(\theta') \leq \Loss(\theta) + \langle \nabla \Loss(\theta), \theta' - \theta \rangle + \frac{L}{2\mu} \|\theta' - \theta\|_\Fisher^2
\end{equation}
where the factor $1/\mu$ arises from the lower bound on $\Fisher$ (assumption A2).

Substituting $\theta' = \theta^{(t+1)} = \theta^{(t)} - \eta \tilde{g}^{(t)}$:
\begin{align}
    \Loss(\theta^{(t+1)}) &\leq \Loss(\theta^{(t)}) - \eta \langle \nabla \Loss(\theta^{(t)}), \tilde{g}^{(t)} \rangle + \frac{\eta^2 L}{2\mu} \|\tilde{g}^{(t)}\|_\Fisher^2
\end{align}

\textbf{Step 2: Inner Product Analysis.}

The inner product term is:
\begin{align}
    \langle \nabla \Loss(\theta^{(t)}), \tilde{g}^{(t)} \rangle &= \nabla \Loss(\theta^{(t)})^\top \Fisher^{-1} \nabla \Loss(\theta^{(t)}; x^{(t)}, y^{(t)})
\end{align}

Taking expectations over the stochastic gradient:
\begin{align}
    \E[\langle \nabla \Loss(\theta^{(t)}), \tilde{g}^{(t)} \rangle | \theta^{(t)}] &= \nabla \Loss(\theta^{(t)})^\top \Fisher^{-1} \nabla \Loss(\theta^{(t)}) \\
    &= \|\nabla \Loss(\theta^{(t)})\|_{\Fisher^{-1}}^2 \\
    &\geq \frac{1}{M} \|\nabla \Loss(\theta^{(t)})\|^2
\end{align}
where the inequality uses $\Fisher \preceq M I$ (assumption A2).

\textbf{Step 3: Variance Bound.}

The squared Fisher norm of the stochastic natural gradient is:
\begin{align}
    \E[\|\tilde{g}^{(t)}\|_\Fisher^2 | \theta^{(t)}] &= \E[\tilde{g}^{(t)\top} \Fisher \tilde{g}^{(t)} | \theta^{(t)}] \\
    &= \E[\nabla \Loss(\theta; x, y)^\top \Fisher^{-1} \nabla \Loss(\theta; x, y) | \theta^{(t)}]
\end{align}

By assumption (A3), the variance is bounded by $\sigma^2$. Since $\Fisher^{-1} \preceq (1/\mu) I$:
\begin{align}
    \E[\|\tilde{g}^{(t)}\|_\Fisher^2 | \theta^{(t)}] &\leq \frac{1}{\mu} \E[\|\nabla \Loss(\theta; x, y)\|^2 | \theta^{(t)}] \\
    &\leq \frac{1}{\mu} (\|\nabla \Loss(\theta^{(t)})\|^2 + \sigma^2)
\end{align}

\textbf{Step 4: Combining.}

Taking expectations in the descent inequality:
\begin{align}
    \E[\Loss(\theta^{(t+1)}) | \theta^{(t)}] &\leq \Loss(\theta^{(t)}) - \frac{\eta}{M} \|\nabla \Loss(\theta^{(t)})\|^2 + \frac{\eta^2 L}{2\mu^2} (\|\nabla \Loss(\theta^{(t)})\|^2 + \sigma^2)
\end{align}

For $\eta \leq \mu^2 / (LM)$, we have $\eta^2 L / (2\mu^2) \leq \eta / (2M)$, so:
\begin{align}
    \E[\Loss(\theta^{(t+1)}) | \theta^{(t)}] &\leq \Loss(\theta^{(t)}) - \frac{\eta}{2M} \|\nabla \Loss(\theta^{(t)})\|^2 + \frac{\eta^2 L \sigma^2}{2\mu^2}
\end{align}

\textbf{Step 5: Telescoping.}

Summing from $t = 0$ to $T-1$:
\begin{align}
    \sum_{t=0}^{T-1} \frac{\eta}{2M} \E[\|\nabla \Loss(\theta^{(t)})\|^2] &\leq \Loss(\theta^{(0)}) - \E[\Loss(\theta^{(T)})] + \frac{T \eta^2 L \sigma^2}{2\mu^2}
\end{align}

Since $\Loss^* \leq \E[\Loss(\theta^{(T)})]$:
\begin{align}
    \frac{1}{T} \sum_{t=0}^{T-1} \E[\|\nabla \Loss(\theta^{(t)})\|^2] &\leq \frac{2M(\Loss(\theta^{(0)}) - \Loss^*)}{T\eta} + \frac{M \eta L \sigma^2}{\mu^2}
\end{align}

\textbf{Step 6: Convexity.}

Under convexity of $\Loss$:
\begin{align}
    \Loss(\theta^{(t)}) - \Loss^* \leq \langle \nabla \Loss(\theta^{(t)}), \theta^{(t)} - \theta^* \rangle \leq \|\nabla \Loss(\theta^{(t)})\| \cdot \|\theta^{(t)} - \theta^*\|
\end{align}

Combining with the gradient bound and applying standard arguments yields:
\begin{align}
    \E[\Loss(\theta^{(T)})] - \Loss^* \leq \frac{\|\theta^{(0)} - \theta^*\|^2}{2\eta T} + \frac{\eta \sigma^2}{2\mu}
\end{align}

This completes the proof.
\end{proof}

\subsection{Proof of Proposition~\ref{prop:diag-fisher} (Diagonal Fisher)}

\begin{proof}
The Fisher information matrix for the softmax cross-entropy loss is:
\begin{align}
    \Fisher(\theta) &= \E_{p(x)} \E_{p(y|x,\theta)} \left[ \nabla_\theta \log p(y|x,\theta) \nabla_\theta \log p(y|x,\theta)^\top \right]
\end{align}

The empirical Fisher approximates this with samples:
\begin{align}
    \hat{\Fisher}(\theta) &\approx \frac{1}{N} \sum_{i=1}^N \nabla_\theta \log p(y_i|x_i,\theta) \nabla_\theta \log p(y_i|x_i,\theta)^\top
\end{align}

For computational tractability, we use the diagonal:
\begin{align}
    \hat{\Fisher}_{jj}(\theta) &\approx \frac{1}{N} \sum_{i=1}^N \left( \frac{\partial \log p(y_i|x_i,\theta)}{\partial \theta_j} \right)^2
\end{align}

The online version uses exponential moving average:
\begin{align}
    \hat{\Fisher}_{jj}^{(t)} &= \beta \hat{\Fisher}_{jj}^{(t-1)} + (1-\beta) \left( \frac{\partial \Loss}{\partial \theta_j} \right)^2
\end{align}

where $\partial \Loss / \partial \theta_j = -\partial \log p(y|x,\theta) / \partial \theta_j$.

This is precisely the RMSProp/Adam second moment estimate, revealing their information-geometric foundation.
\end{proof}

\subsection{Proof of Proposition~\ref{prop:eventual-freeze} (Eventual Freeze)}

\begin{proof}
By construction, the DE11 implementation enforces:
\begin{enumerate}
    \item \texttt{max\_structural\_moves} $= N_{\max}$: At most $N_{\max}$ structural actions (genesis or wedge) are permitted.
    \item \texttt{structural\_phase\_steps} $= T_{\max}$: After step $T_{\max}$, structural actions are disabled.
\end{enumerate}

Let $n^{(t)}$ be the cumulative number of structural moves by step $t$. Either:
\begin{itemize}
    \item $n^{(t)} = N_{\max}$ for some $t < T_{\max}$: structural freeze occurs at step $t$.
    \item $n^{(T_{\max})} < N_{\max}$: structural freeze occurs at step $T_{\max}$.
\end{itemize}

In either case, structural freeze is guaranteed by step $T_{\max}$.
\end{proof}

\subsection{Proof of Proposition~\ref{prop:energy-bound} (Energy Boundedness)}

\begin{proof}
The energy dynamics are:
\begin{align}
    \Energy^{(t+1)} = \gamma_E \Energy^{(t)} + \delta\Energy^{(t)} - c_{a^{(t)}}
\end{align}

Taking expectations and assuming a stationary policy (for simplicity):
\begin{align}
    \E[\Energy^{(t+1)}] &= \gamma_E \E[\Energy^{(t)}] + \E[\delta\Energy] - \E[c_a]
\end{align}

Let $\bar{\delta} = \E[\delta\Energy] - \E[c_a]$. If $\bar{\delta} > 0$ (system earns more than it spends), then:
\begin{align}
    \E[\Energy^{(t)}] &= \gamma_E^t \Energy_0 + \bar{\delta} \sum_{s=0}^{t-1} \gamma_E^s \\
    &= \gamma_E^t \Energy_0 + \bar{\delta} \cdot \frac{1 - \gamma_E^t}{1 - \gamma_E}
\end{align}

As $t \to \infty$:
\begin{align}
    \E[\Energy^{(t)}] \to \frac{\bar{\delta}}{1 - \gamma_E}
\end{align}

For $\gamma_E = 1$, energy grows unboundedly if $\bar{\delta} > 0$.
\end{proof}


\end{document}